\documentclass[11pt]{article}
\usepackage[letterpaper,margin=1in]{geometry}
\usepackage{times}
\usepackage{natbib}
\usepackage{amsmath,amssymb,amsthm,booktabs,multirow,makecell,tabularx,array}
\usepackage{algorithm}
\usepackage[noend]{algpseudocode}
\usepackage{graphicx}
\usepackage{float}
\usepackage{placeins}
\usepackage{dblfloatfix}
\usepackage{hyperref}
\usepackage{url}
\usepackage[table]{xcolor}
\hypersetup{hidelinks}
\renewcommand{\dbltopfraction}{.95}

\renewcommand{\dblfloatpagefraction}{.80}

\usepackage{amsmath,amsfonts,bm}

\def\eqref#1{equation~\ref{#1}}

\def\1{\bm{1}}

\DeclareMathAlphabet{\mathsfit}{\encodingdefault}{\sfdefault}{m}{sl}
\SetMathAlphabet{\mathsfit}{bold}{\encodingdefault}{\sfdefault}{bx}{n}

\newcommand{\E}{\mathbb{E}}

\newcommand{\method}{ToPO}
\newcommand{\LossTopo}{\mathcal{L}_{\mathrm{ToPO}}}
\newcommand{\LossTri}{\mathcal{L}_{\mathrm{tri}}}

\newcommand{\norm}[1]{\left\lVert #1 \right\rVert}
\newcommand{\inner}[2]{\left\langle #1,#2 \right\rangle}

\definecolor{TopRankMint}{HTML}{D1F5E6}
\definecolor{RunnerUpMist}{HTML}{EAF0F7}
\newcommand{\best}[1]{\cellcolor{TopRankMint}\textbf{#1}}
\newcommand{\runnerup}[1]{\cellcolor{RunnerUpMist}\textit{#1}}

\newtheorem{theorem}{Theorem}

\newtheorem{proposition}[theorem]{Proposition}

\newcommand{\arxivTitle}{ToPO: Token-Oriented Preference Optimization\\for Text-to-Image Diffusion Models}

\newcommand{\arxivAuthors}{%
  \begin{tabular}{cc}
    Juntao Xu$^{1,*}$ & Shihong Li$^{2,*}$ \\
    \small\texttt{xjd26@mails.tsinghua.edu.cn} &
    \small\texttt{202522280709@std.uestc.edu.cn} \\[4pt]
    Hoi Fan Au$^{1}$ & Ning Zhu$^{3}$ \\
    \small\texttt{okx26@mails.tsinghua.edu.cn} &
    \small\texttt{ningzhu@stanford.edu}
  \end{tabular}\\[6pt]
  \small $^{1}$Tsinghua University \qquad
  $^{2}$University of Electronic Science and Technology of China\\
  \small $^{3}$Stanford University \qquad $^{*}$Equal contribution.%
}

\title{\arxivTitle}
\author{\arxivAuthors}
\date{}
\hypersetup{pdftitle={ToPO: Token-Oriented Preference Optimization for Text-to-Image Diffusion Models},
  pdfauthor={Juntao Xu, Shihong Li, Hoi Fan Au, Ning Zhu}}

\begin{document}
\maketitle

\vspace{-0.45in}
\begin{figure}[H]
\centering
\includegraphics[width=.99\textwidth]{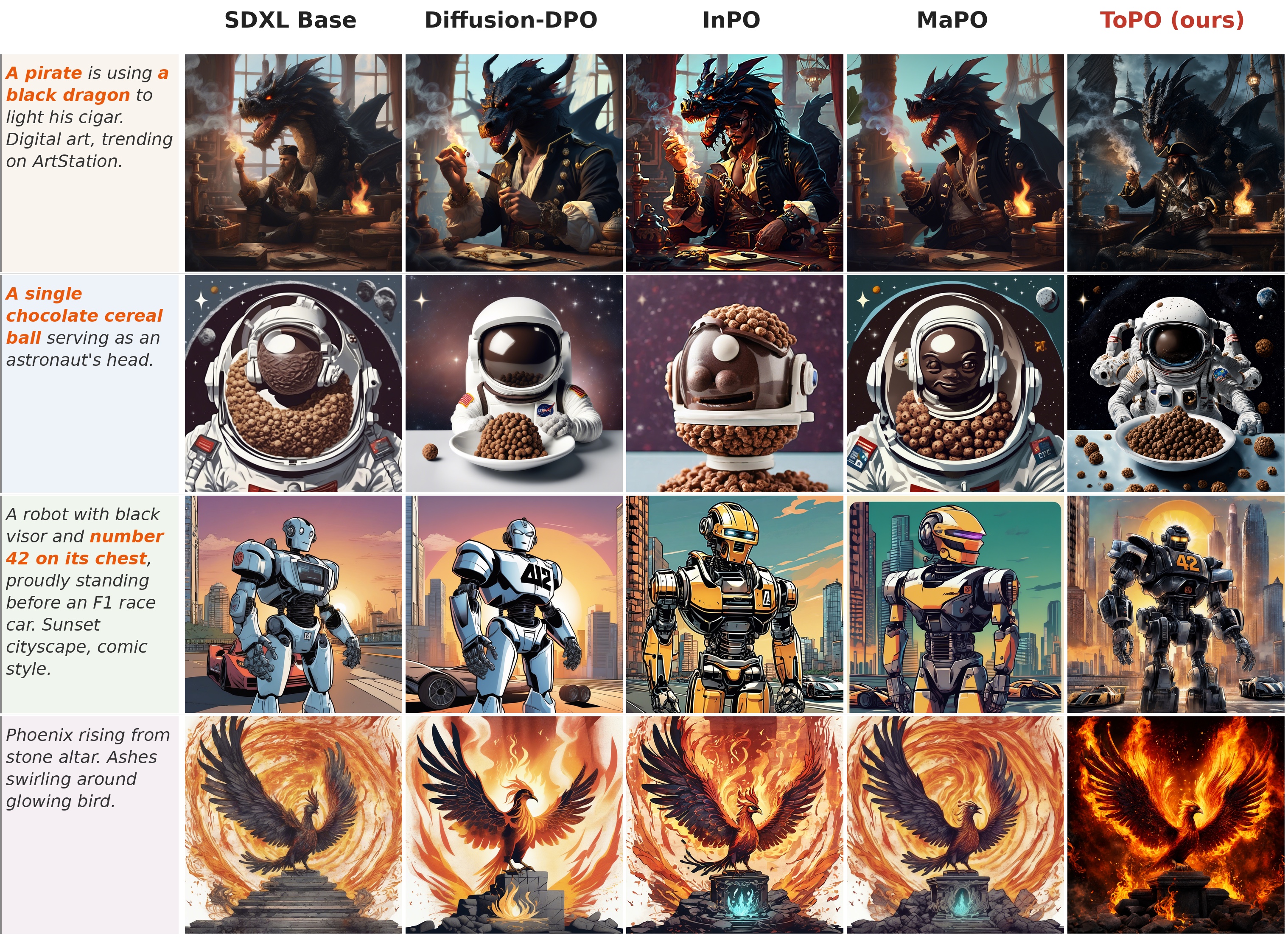}
\caption{\textbf{Qualitative comparison on four fixed SDXL prompts.} Columns show SDXL Base, Diffusion-DPO, InPO, MaPO, and ToPO. Orange bold prompt terms mark the target constraints.}
\label{fig:qualitative-cover}
\end{figure}
\begin{abstract}
Diffusion models have achieved remarkable success in text-to-image generation, yet aligning their outputs with human preferences remains challenging. Diffusion-DPO learns from image preference pairs, but their global labels do not specify how supervision should be allocated across spatial locations and denoising timesteps. The semantic roles of prompt tokens are also not explicitly considered when assigning these local weights. To address these limitations, we propose Token-Oriented Preference Optimization (\method{}), which builds a token-conditioned spatial--temporal weighting map for each pair. A frozen reference denoiser obtains spatial and temporal weights from local noise-prediction error differences under shared noise, while preferred-branch cross-attention links spatial evidence to content tokens. The map reweights Diffusion-DPO and supports an auxiliary pixel-midpoint ordering loss. It requires neither local preference annotations nor a separately trained reward model. Extensive experiments across multiple preference and compositional evaluations show that \method{} achieves state-of-the-art performance among the compared methods with only 500 training updates, demonstrating the potential of token-conditioned preference allocation for high-quality diffusion alignment within a limited update budget.
\end{abstract}

\section{Introduction}

\begingroup
\setlength{\parskip}{0.85ex plus 0.15ex minus 0.1ex}
\setlength{\parindent}{0pt}

\makeatletter
\renewcommand{\paragraph}{\@startsection{paragraph}{4}{\z@}
  {0.85ex plus 0.15ex minus 0.1ex}
  {-0.55em}
  {\normalsize\bfseries}}
\makeatother

Text-to-image diffusion models now generate high-fidelity images from natural-language prompts~\citep{rombach2022ldm}.  Yet their outputs still fall short of human preferences, both in visual quality and in faithfulness to the prompt, especially when a prompt specifies several objects, attributes, or spatial relations~\citep{ghosh2023geneval}.  Reinforcement learning and reward backpropagation can narrow this gap~\citep{black2024ddpo,prabhudesai2023alignprop}, but they rely on a separately trained reward model and repeated online sampling.  Diffusion-DPO~\citep{wallace2024diffusiondpo} removes both requirements by extending Direct Preference Optimization~\citep{rafailov2023dpo} to pairs of preferred and dispreferred images, and has become a standard basis for offline preference alignment~\citep{li2024diffusionkto,zhu2025dspo,kang2026rethinking}.

However, this simplicity comes with a coarse view of supervision.  Each pair contributes a single image-level comparison, and Diffusion-DPO spreads it evenly over every latent location and every denoising timestep.  Yet the reason one image wins is rarely spread evenly, in either space or time.  Consider two images generated for the prompt ``a robot with black visor and number 42 on its chest'' (Figure~\ref{fig:qualitative-cover}).  One may win simply because the number on its chest is legible, a difference confined to a small region, while the two images also differ in background, pose, and lighting that the prompt never mentions.  The decisive difference is confined in time as well: global layout and pose form at early, high-noise steps, whereas fine details such as the digits take shape at later steps~\citep{balaji2022ediff,choi2022perception}.  Even allocation therefore spends much of each update on regions and timesteps that have little to do with why one image won, diluting the supervision carried by each pair.  What each comparison lacks is an account of \emph{where} and \emph{when} it applies.

Recent work supplies this account, but draws it from outside the preference pair.  Spatial methods assign credit to patches, instances, or regions~\citep{huang2025patchdpo,sun2026iapo,liu2026bidpo}, relying on pretrained vision models, object-level annotations, or masks to decide which regions matter.  Temporal methods score individual denoising steps with step-aware reward models trained on noisy latents~\citep{liang2025spo,zhang2025lpo,xian2026tapo}, which brings back the learned reward model that Diffusion-DPO was designed to remove.  Each line also covers only one of the two axes, and allocating credit over both has so far relied on online reinforcement learning with a reward model~\citep{shen2026star}.  This raises a question: \textbf{can a preference pair itself tell \emph{where} and \emph{when} its supervision should go?}

\begingroup
\setlength{\parskip}{0pt}
\setlength{\intextsep}{3pt plus 1pt minus 1pt}
\begin{figure}[H]
\centering
\includegraphics[width=.98\textwidth]{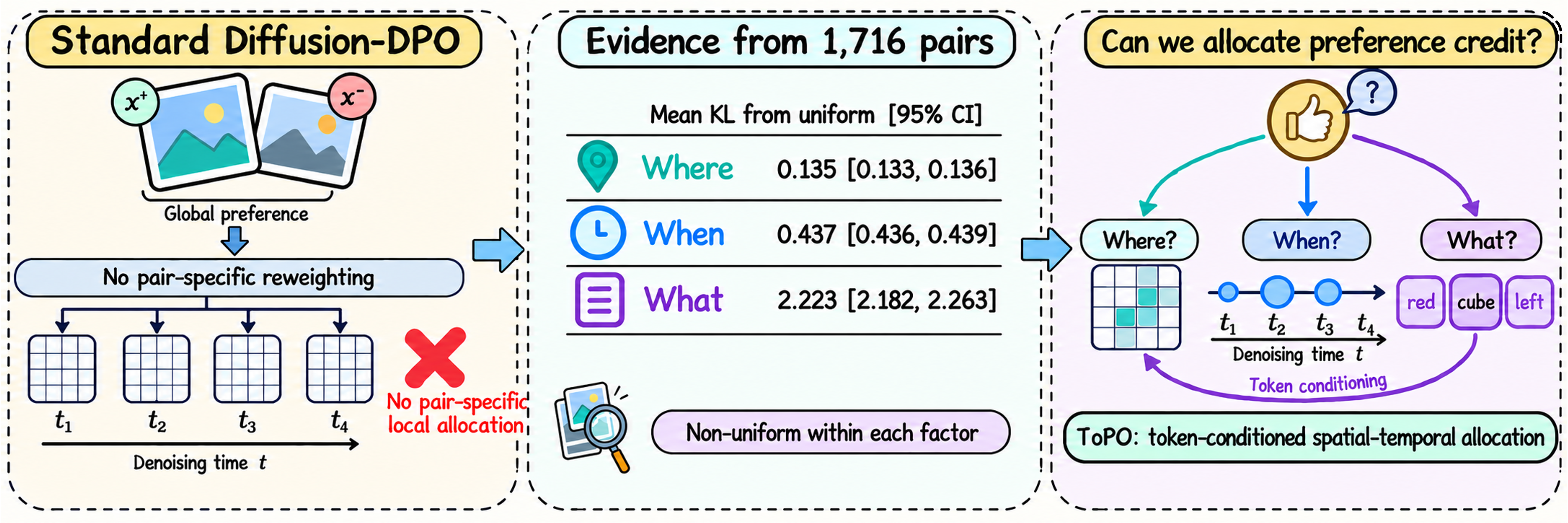}
\begingroup
\setlength{\abovecaptionskip}{2pt}
\setlength{\belowcaptionskip}{0pt}
\caption{\textbf{Motivation for token-conditioned preference allocation.}  Left: Diffusion-DPO spreads each image-level comparison evenly over latent locations and denoising timesteps.  Middle: over 1,716 preference pairs, the weights derived from the reference error gap and cross-attention depart from uniform in \emph{where}, \emph{when}, and \emph{what} (mean KL from uniform with 95\% confidence intervals).  Right: \method{} allocates each comparison by \emph{where}, \emph{when}, and \emph{what} the pair disagrees about.}
\endgroup
\label{fig:g0-credit}
\end{figure}
\endgroup

The answer can come from a component that Diffusion-DPO already has: its frozen reference denoiser.  When the two images are corrupted with the same noise, any gap between the reference model's prediction errors on them comes from the images themselves.  Measured at each location and timestep, this error gap shows where and when the model responds to the two images differently.  The error gap does not tell which image is better, and it does not need to: the preference label already fixes the direction of the update, so the error gap only has to decide \emph{where} and \emph{when} the update lands.  Across 1,716 Pick-a-Pic pairs, the weights it produces are clearly non-uniform in both space and time (Figure~\ref{fig:g0-credit}).  The error gap, however, cannot tell which differences the prompt cares about.  Two images sampled independently for the same prompt differ almost everywhere, so the error gap highlights every difference the model responds to, whether or not the prompt mentions it.  A change in background texture can weigh as much as the difference between a legible and an illegible ``42''.

Separating these differences requires knowing \emph{what} each one is about, and the same reference denoiser already encodes this in its cross-attention, which ties each prompt token to the image regions that render it~\citep{hertz2023prompttoprompt,chefer2023attendandexcite}.  Reading the error gap through these attention maps scores each content token by how much of the gap falls on the regions that attend to it.  Because the scores come from the error gap, they favor the tokens on which the two images actually disagree, such as ``42'' in our example, rather than every token in the prompt.  Projecting the scores back onto the image then raises the weight of regions tied to these tokens, while regions whose attention falls on no content token, such as the background texture, lose weight.  In a controlled test, target color, shape, and texture tokens receive 26\% to 40\% more weight than matched non-target tokens (Appendix~\ref{app:token-details}).

We turn these observations into Token-Oriented Preference Optimization (\method{}), which reweights the Diffusion-DPO loss by \emph{where}, \emph{when}, and \emph{what} each pair disagrees about.  \method{} has two components.  \textbf{(1) Error-gap map.}  The error gap of the frozen reference denoiser yields a weight for each latent location and a weight for each denoising timestep.  \textbf{(2) Token conditioning.}  Cross-attention on the preferred image, which renders the prompt in the way the rater favored, converts the spatial weights into token scores and projects them back, so that the spatial weights concentrate on differences tied to the prompt.  In the loss, the spatial weights set how much each location contributes to the pairwise comparison, and the timestep weight sets how strongly the comparison is enforced at the sampled step.  Because the weights come from a frozen model and carry no gradient, \method{} leaves the Diffusion-DPO pipeline unchanged apart from the weights inside its loss.  It thus answers \emph{where}, \emph{when}, and \emph{what} from the preference pair itself, with no local annotation and no reward model.

Our contributions are summarized as follows:
\vspace{-3pt}
\begin{itemize}
    \setlength{\itemsep}{1pt}
    \setlength{\topsep}{1pt}
    \setlength{\parsep}{0pt}
    \item We separate the direction of a preference update from its allocation.  The pairwise label fixes the direction, and we show that the allocation over \emph{where}, \emph{when}, and \emph{what} can be derived from the preference pair itself.
    \item We propose \method{}, a drop-in reweighting of Diffusion-DPO built from the error gap of its frozen reference denoiser and the cross-attention of the preferred image.
    \item Under a matched training budget on SD-1.5 and SDXL, \method{} improves over Diffusion-DPO on preference and compositional benchmarks.  
    In a blind study on SDXL, raters prefer \method{} over Diffusion-DPO in 63.3\% to 70.7\% of comparisons.  Controls that flatten or shuffle the weights indicate that both their non-uniformity and their placement matter.
\end{itemize}

\endgroup

\section{Related Work}
\label{sec:related}

\begingroup
\setlength{\parskip}{0pt}
\setlength{\parindent}{0pt}
\makeatletter
\renewcommand{\paragraph}{\@startsection{paragraph}{4}{\z@}{0.85ex plus 0.15ex minus 0.1ex}{-0.55em}{\normalsize\bfseries}}
\makeatother

\paragraph{Preference alignment for diffusion models.}
DPO learns from paired preferences without an explicit reward model or an online reinforcement-learning loop \citep{rafailov2023dpo}. Diffusion alignment has explored reward-driven trajectory training \citep{black2024ddpo}, pairwise objectives such as D3PO and Diffusion-DPO \citep{yang2024d3po,wallace2024diffusiondpo}, per-image utility optimization in Diffusion-KTO \citep{li2024diffusionkto}, and score-based preference learning \citep{zhu2025dspo}. Other work addresses noisy preference pairs and timestep or reference choices \citep{li2026alphadpo,liu2026semidpo,kang2026rethinking}. These methods mainly optimize image-level preferences and leave the allocation of supervision across local denoising coordinates unspecified.

\paragraph{Fine-grained spatial and instance alignment.}
Prior work has refined visual preference learning at different levels of granularity, from image regions to entire samples. PatchDPO introduces local patch supervision, while IAPO and BiDPO construct instance- or region-level preferences from image comparisons \citep{huang2025patchdpo,sun2026iapo,liu2026bidpo}. Visual preference optimization methods further provide pixel- or sample-level supervision, and OSPO uses object-centric data and attention masks \citep{ni2026vipo,li2026viposcale,oh2026ospo}. These methods demonstrate the value of localized credit, whereas \method{} derives a spatial route directly from the supplied preferred and dispreferred images.

\paragraph{Temporal and latent preference optimization.}
Trajectory-aware methods distribute preference supervision over denoising time or latent space through dense rewards, DDIM inversion, InPO, and related objectives \citep{yang2024densereward,lu2025inpo,li2025inversiondpo,wu2025rdpo}. SPO, LPO, and SLRM/TAPO further model step-aware or noise-compatible latent rewards \citep{liang2025spo,zhang2025lpo,xian2026tapo}. These approaches motivate temporal credit assignment, while \method{} infers a separable time factor directly from frozen-reference residual contrast.

\paragraph{Attention-guided routing and control.}
Cross-attention connects text tokens with latent image features in diffusion models \citep{rombach2022ldm}. Prompt-to-Prompt and Attend-and-Excite use attention maps for inference-time editing and semantic control \citep{hertz2023prompttoprompt,chefer2023attendandexcite}, while STAR studies adaptive spatiotemporal reward allocation for text-to-image reinforcement learning \citep{shen2026star}. \method{} uses preferred-branch cross-attention to convert pair-derived spatial evidence into token-conditioned DPO weights.

\smallskip
\noindent
Taken together, prior work studies image-level preference optimization, spatial credit, temporal allocation, and attention control largely as separate problems. \method{} combines these factors in one offline objective by deriving a spatial-temporal route from frozen-reference residual evidence and conditioning it on content tokens through preferred-branch cross-attention.

\endgroup
\section{Preliminaries: Diffusion Preference Learning}
\label{sec:prelim}
\begingroup
\setlength{\parskip}{0pt}
\setlength{\parindent}{0pt}
\setlength{\abovedisplayskip}{8pt plus 2pt minus 2pt}
\setlength{\belowdisplayskip}{8pt plus 2pt minus 2pt}
\setlength{\abovedisplayshortskip}{6pt plus 2pt minus 1pt}
\setlength{\belowdisplayshortskip}{6pt plus 2pt minus 1pt}
\makeatletter
\renewcommand{\paragraph}{\@startsection{paragraph}{4}{\z@}{0.85ex plus 0.15ex minus 0.1ex}{-0.55em}{\normalsize\bfseries}}
\makeatother

\subsection{Conditional latent diffusion}
Let $x$ be an image, $z_0=\mathcal E(x)$ its latent encoding, and $c$ a text prompt. For noise $\varepsilon\sim\mathcal N(0,I)$, the DDPM forward process \citep{ho2020ddpm} produces
\begin{equation}
z_\tau=\sqrt{\bar\alpha_\tau}z_0+\sqrt{1-\bar\alpha_\tau}\,\varepsilon,
\qquad \bar\alpha_\tau=\prod_{j\leq\tau}(1-\beta_j).
\label{eq:forward-diffusion}
\end{equation}
The denoiser $\varepsilon_\theta(z_\tau,\tau,c)$ predicts the added noise. Each prediction is indexed by a latent spatial location $s$ and a denoising timestep $\tau$. At each location and timestep, we define the local residual as
\begin{equation}
d_\theta(x;c,s,\tau)=
\norm{\varepsilon_\theta(z_\tau,\tau,c)[s]-\varepsilon[s]}_2^2.
\label{eq:local-residual}
\end{equation}
This residual is the basic quantity that a route can allocate across diffusion coordinates. Prompt tokens affect these coordinates through cross-attention; they do not index noise-prediction residuals directly.

\subsection{Human preferences and Diffusion-DPO}
The preference data consist of triples $Z=(c,x^+,x^-)$, where $x^+$ is preferred to $x^-$ for prompt $c$. Under the Bradley--Terry model,
\begin{equation}
\Pr(x^+\succ x^-\mid c)=
\sigma\!\left(r^*(c,x^+)-r^*(c,x^-)\right),
\label{eq:bradley-terry}
\end{equation}
where $r^*$ is an unobserved human reward. Direct preference optimization compares a learned policy score with a reference score, without fitting an explicit reward model \citep{rafailov2023dpo}. Diffusion-DPO instantiates this comparison with a tractable surrogate for the image-level diffusion likelihood \citep{wallace2024diffusiondpo}. Let $S_m(c,x)$ denote this surrogate for $m\in\{\theta,\mathrm{ref}\}$, and define $\Delta_m(Z)=S_m(c,x^+)-S_m(c,x^-)$. The pairwise objective is
\begin{equation}
\mathcal L_{\mathrm{D\text{-}DPO}}
=-\E_Z\log\sigma\!\left(
\beta\,[\Delta_\theta(Z)-\Delta_{\mathrm{ref}}(Z)]\right).
\label{eq:diffusion-dpo-prelim}
\end{equation}
The trajectory surrogate varies across diffusion preference objectives. In each case, a preference pair enters the loss through one image-level scalar comparison.
\endgroup
\section{ToPO: Token-Oriented Preference Optimization}
\label{sec:method}

\begingroup
\setlength{\parskip}{0pt}
\setlength{\parindent}{0pt}
\setlength{\abovedisplayskip}{8pt plus 2pt minus 2pt}
\setlength{\belowdisplayskip}{8pt plus 2pt minus 2pt}
\setlength{\abovedisplayshortskip}{6pt plus 2pt minus 1pt}
\setlength{\belowdisplayshortskip}{6pt plus 2pt minus 1pt}
\makeatletter
\renewcommand{\paragraph}{\@startsection{paragraph}{4}{\z@}{0.85ex plus 0.15ex minus 0.1ex}{-0.55em}{\normalsize\bfseries}}
\makeatother

\subsection{Overview}

\method{} constructs a token-conditioned route over native diffusion coordinates from each preference pair.  A frozen reference first evaluates both branches with shared noise at four anchors and channel-averages their local squared-residual contrast.  Averaging this contrast over anchors yields the spatial factor $w_s(s)$; averaging it over latent sites yields discrete anchor-time evidence, which Gaussian interpolation turns into $w_\tau(\tau)$ at the sampled training timestep.  Preferred-branch cross-attention reads the pre-prior spatial evidence into content-token weights $w_k(k)$ and projects their normalized, prior-mixed values back to the latent grid.  This token-induced spatial map modulates $w_s(s)$ to form the separable route $W(s,\tau)$.  Thus, tokens modulate a diffusion-native route rather than defining a separate loss coordinate.  The frozen reference constructs the detached route, whose spatial factor also weights the midpoint regularizer.  Figure~\ref{fig:topo-pipeline} summarizes the training-time computation.

\begin{figure*}[t]
\centering
\includegraphics[width=.99\textwidth]{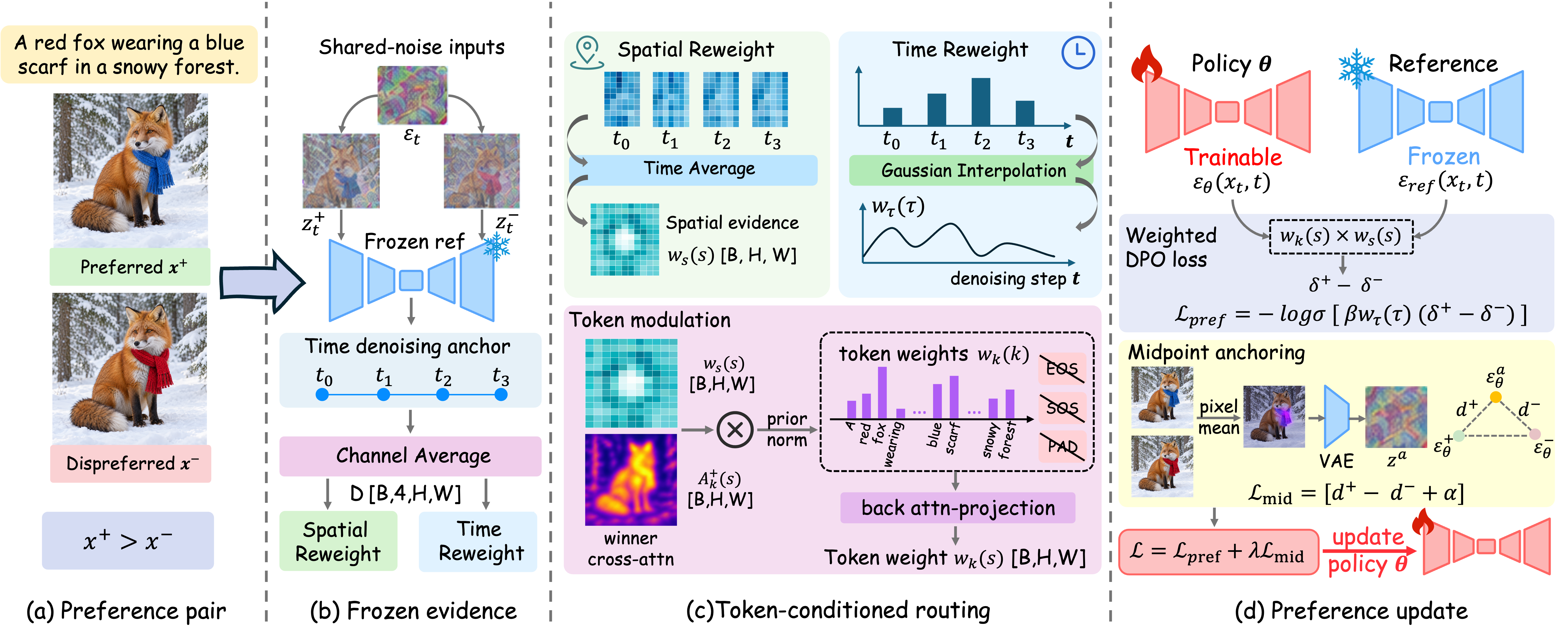}
\caption{\textbf{ToPO routing pipeline.}  A preferred--dispreferred pair is evaluated by a frozen reference under shared noise at four denoising anchors.  Channel-averaged local residual contrasts are averaged over anchors to form spatial evidence and over latent sites to form discrete timestep evidence; normalized Gaussian interpolation evaluates the timestep factor at the sampled training step.  Preferred-branch cross-attention reads the pre-prior spatial evidence into content-token weights.  After content-token masking, prior mixing, and normalization, the token weights are backprojected through the same attention maps to modulate the spatial factor.  The resulting detached spatial--temporal route reweights the Diffusion-DPO loss, and its spatial component also weights the pixel-midpoint ordering term.  Pair labels specify the preferred direction; the route allocates the update over latent sites and denoising time.}
\label{fig:topo-pipeline}
\end{figure*}

\subsection{Token-conditioned separable spatial--temporal route construction}

For an anchor $\tau_i\in\{50,349,649,949\}$, the residual-contrast pass noises both endpoints with the same draw $\varepsilon_i$ and evaluates the frozen-reference predictions $\widehat\varepsilon_{\mathrm{ref},i}^{\pm}=\varepsilon_{\mathrm{ref}}(z_{\tau_i}^{\pm},\tau_i,c)$.  For channel $q$ and latent site $s$, its branchwise squared-residual contrast is
\begin{equation}
 R_i^{\pm}(q,s)=\omega_i\bigl(\varepsilon_i(q,s)-\widehat\varepsilon_{\mathrm{ref},i}^{\pm}(q,s)\bigr)^2,
 \qquad
 D_i(s)=\frac{1}{C}\sum_{q=1}^{C}\left|R_i^+(q,s)-R_i^-(q,s)\right| .
 \label{eq:residual-contrast}
\end{equation}
Here $\omega_i$ is the Diffusion-DPO integrand weight.  $D_i(s)$ is a symmetric observable residual-contrast magnitude: it is high where the frozen reference exhibits different local residual magnitudes for the two image branches.  We use this label-invariant quantity to assign \emph{allocation mass}, while the signed winner--loser term $\delta^+-\delta^-$ in \eqref{eq:topo-loss} specifies preference direction.  For example, a site with residual magnitudes $(0.1,0.4)$ receives contrast $0.3$ under either ordering of the branch labels.  Thus the route emphasizes coordinates with larger branchwise residual-magnitude differences; it is not a standalone local preference label.  Let $\operatorname{Norm}_1$ normalize a nonnegative vector over its declared support to sum to one.  The residual-contrast pass forms the data-dependent spatial and anchor-time marginals
\begin{equation}
 d_s(s)=\frac1{|\mathcal A|}\sum_{i\in\mathcal A}D_i(s),\qquad
 d_\tau(\tau_i)=\frac1{|\mathcal S|}\sum_{s\in\mathcal S}D_i(s),
 \label{eq:credit-marginals}
\end{equation}
then mixes each normalized marginal with its declared structural prior:
\begin{equation}
 w_s=\operatorname{Norm}_1\!\left[\lambda\operatorname{Norm}_1[d_s]+(1-\lambda)p_s\right],\qquad
 \bar w_\tau(\tau_i)=\operatorname{Norm}_1\!\left[\lambda\operatorname{Norm}_1[d_\tau]+(1-\lambda)p_\tau\right].
 \label{eq:spatial-temporal-credit}
\end{equation}
The locked configuration uses $\lambda=0.5$, a uniform spatial prior, and an SNR-flat anchor-time prior.  During training, the anchor-time distribution is evaluated at an arbitrary sampled timestep through normalized Gaussian interpolation with $\sigma=100$:
\begin{equation}
 w_\tau(\tau)=|\mathcal A|\sum_{i\in\mathcal A}
 \kappa_\sigma(\tau-\tau_i)\,\bar w_\tau(\tau_i),
 \qquad \kappa_\sigma(\tau-\tau_i)=
 \frac{\exp(-\lvert\tau-\tau_i\rvert^2/(2\sigma^2))}
 {\sum_{j\in\mathcal A}\exp(-\lvert\tau-\tau_j\rvert^2/(2\sigma^2))}.
 \label{eq:timestep-credit}
\end{equation}

At every preferred-branch residual-contrast forward, we capture the attention probabilities of all U-Net modules named \texttt{attn2}.  We average query--key attention probabilities over heads, average each captured layer over the four anchors, average the resized layer maps, and bilinearly resize the result to the native latent grid ($64\times64$ for SD-1.5 and $128\times128$ for SDXL); denote the resulting map by $\bar A_k^+(s)$.  We use the preferred branch as the declared coordinate-to-token readout: it conditions the spatial route on the image that defines the desired direction, while the signed DPO logit in \eqref{eq:topo-loss} alone carries the winner--loser orientation.  This choice is neither a target annotation nor a claim that preferred-only attention is the unique valid bridge; loser-only and symmetric bridges are distinct route variants.  The token pathway reads the pre-prior spatial evidence $w_s^{\mathrm{data}}=\operatorname{Norm}_1[d_s]$ through this map:
\begin{equation}
  w_k^{\mathrm{data}}(k)\ \propto
  \mathbf 1_{\mathrm{content}}(k)\sum_s \bar A_k^+(s) w_s^{\mathrm{data}}(s),
  \qquad
  w_k=\operatorname{Norm}_1\!\left[\lambda\operatorname{Norm}_1[w_k^{\mathrm{data}}]+(1-\lambda)p_k\right].
  \label{eq:token-modulation}
\end{equation}
The content-token indicator removes padding, SOS, and the valid EOS position before normalization.  This defines $w_k$ as an attention-weighted token-conditioned modulation statistic of branchwise residual disagreement while retaining a content-token prior.

Cross-attention folds the factors into the field used by the objective:
\begin{equation}
 \begin{aligned}
 W_s(s)&=\operatorname{Norm}_{\mathrm{mean}}\!\left(
 w_s(s)\left[\sum_k \bar A_k^+(s)w_k(k)\right]\right),\\[-2pt]
 W(s,\tau)&=W_s(s)w_\tau(\tau),
 \qquad (\widetilde W_s,\widetilde w_\tau)=\operatorname{sg}[W_s,w_\tau],\\[-2pt]
 \widetilde W(s,\tau)&=\widetilde W_s(s)\widetilde w_\tau(\tau).
 \end{aligned}
 \label{eq:joint-weight}
\end{equation}
Here $w_s\in\mathbb R^{|\mathcal S|}$, $w_\tau\in\mathbb R^{|\mathcal T|}$, and $w_k\in\mathbb R^{|\mathcal K|}$, whereas $W\in\mathbb R^{|\mathcal S|\times|\mathcal T|}$.  Equation~\ref{eq:joint-weight} is deliberately separable over the two native loss coordinates: $w_k$ modulates $W_s$, and does not introduce a three-dimensional $(s,\tau,k)$ loss tensor.  At every update $t$, the frozen reference computes $W_t=\mathcal W_{\mathrm{ref}}(\mathcal B_t)$ from the current preference minibatch.  The policy then uses $\widetilde W_t=\operatorname{sg}[W_t]$ on that same update: the route is differentiated through neither the attributor nor the reference model.  This per-minibatch construction is shared by the SD-1.5 and SDXL implementations.

The normalization convention is part of the construction rather than a cosmetic rescaling.  For a nonnegative map $u$ on its declared support $\mathcal J$, we use
\begin{equation}
 \operatorname{Norm}_{\mathrm{mean}}[u]_j=
 \frac{u_j}{\max\{\lvert\mathcal J\rvert^{-1}\sum_{j'\in\mathcal J}u_{j'},\,10^{-8}\}}.
 \label{eq:mean-normalization}
\end{equation}
Consequently, multiplying the factors changes their arrangement across coordinates without allowing a vanishing or arbitrarily scaled factor to set the global loss scale.  The $10^{-8}$ floor is the implementation safeguard used before the post-clipping update.  This convention also makes the Shuffled-$W$ stress test interpretable: it alters the arrangement of a normalized route rather than merely changing the total weight applied to the loss.

\subsection{Preference and anchor losses}

The main objective applies the detached route to the reference-relative local residual advantage.  With policy and reference predictions $\widehat\varepsilon_\theta^{\pm}$ and $\widehat\varepsilon_{\mathrm{ref}}^{\pm}$ at a sampled training timestep, define
\begin{equation}
 \begin{aligned}
 \delta^{\pm}&=
 \frac{1}{C|\mathcal S|}\sum_{q,s}\widetilde W_s(s)
 \left[
 \bigl(\widehat\varepsilon_{\mathrm{ref}}^{\pm}(q,s)-\varepsilon^{\pm}(q,s)\bigr)^2-
 \bigl(\widehat\varepsilon_{\theta}^{\pm}(q,s)-\varepsilon^{\pm}(q,s)\bigr)^2
 \right],\\[-2pt]
 \LossTopo&=-\E\log\sigma\!\left(\beta\widetilde w_\tau(\tau)[\delta^+-\delta^-]\right).
 \end{aligned}
 \label{eq:topo-loss}
\end{equation}
Thus $\widetilde W_s$ weights the per-site residual reduction and $\widetilde w_\tau$ scales the sampled-timestep preference logit.  We further form a deterministic pixel-midpoint anchor $x^a=\mathcal E((x^++x^-)/2)$ and penalize an ordering violation between the policy's spatially weighted prediction distances to the two endpoints:
\begin{equation}
 \LossTri=\E\left[\max\!\left\{0,
 \norm{(\widehat\varepsilon_\theta^a-\widehat\varepsilon_\theta^+)\odot\sqrt{\widetilde W_s}}_2^2-
 \norm{(\widehat\varepsilon_\theta^a-\widehat\varepsilon_\theta^-)\odot\sqrt{\widetilde W_s}}_2^2+\alpha\right\}\right].
 \label{eq:triplet}
\end{equation}

The complete loss is $\LossTopo+\gamma\LossTri$, with $\gamma=0.1$ and $\alpha=0.05$ in the locked run configuration. We encode the pixel midpoint $\mathcal E((x^++x^-)/2)$ rather than interpolate the two endpoint latents; this fixes a deterministic, symmetric image-space anchor under the same encoder used for both endpoints. Pixel averaging can contain ghosted or otherwise off-manifold content, so the anchor is not treated as a natural-image target and no generated image is trained to imitate it. It only supplies a common reference for the relative ordering of the two spatially weighted prediction distances. Because its three predictions are evaluated at one sampled timestep, it uses the shared factor $\widetilde W_s$, not the full $W(s,\tau)$.

The controlled removals and scope of the supporting characterization are specified in Appendix~\ref{app:training-procedure}. They keep the underlying preference pairs fixed and test sensitivity to the declared routing factors under the locked configuration.
\endgroup
\section{Experiments}
\label{sec:experiments}

\begingroup
\setlength{\parskip}{0pt}
\setlength{\parindent}{0pt}
\makeatletter
\renewcommand{\paragraph}{\@startsection{paragraph}{4}{\z@}{0.85ex plus 0.15ex minus 0.1ex}{-0.55em}{\normalsize\bfseries}}
\makeatother

\subsection{Datasets and implementation details}

We train SD-1.5 and SDXL on Pick-a-Pic v2 \citep{kirstain2023pickapic}.  Removing approximately 12\% tie comparisons following Diffusion-DPO \citep{wallace2024diffusiondpo} leaves 851,293 ordered pairs from 58,960 prompts.  SD-1.5 uses the released SFT-winners checkpoint at $512^2$, while SDXL uses SDXL Base at $1024^2$; in both cases, we finetune the policy U-Net and freeze the text encoder, VAE, and reference U-Net.  All runs use eight NVIDIA H200 GPUs and 500 AdamW updates.  We use a learning rate of $10^{-5}$, a 200-step warmup, weight decay of $0.01$, and gradient clipping at $1.0$.  Per-device batches of 4 and 2 with 32-step accumulation give effective batches of 1,024 and 512, respectively.  SD-1.5 uses fp16 and SDXL bf16, with route construction and master weights in fp32.  We use four route anchors, $\beta=2000$, $\gamma=0.1$, and margin $0.05$; Appendix~\ref{app:training-procedure} gives the remaining settings.

\subsection{Baselines and evaluation details}
\label{sec:eval-details}

\paragraph{Baselines.} For SD-1.5, we compare Base, Diffusion-DPO \citep{wallace2024diffusiondpo}, SFT-winners, KTO \citep{li2024diffusionkto}, InPO \citep{lu2025inpo}, and ToPO; the SDXL comparison replaces KTO with MaPO \citep{hong2024mapo}. The multi-method tables use released checkpoints, while the three-seed ToPO--Diffusion-DPO study matches initialization, data order, optimizer, effective batch, and 500-update schedule (Appendix~\ref{app:seed-stability}).
\paragraph{Benchmarks and metrics.} On 500 Pick-a-Pic validation prompts, we report PickScore, HPSv2 \citep{wu2023hpsv2}, ImageReward \citep{xu2023imagereward}, aesthetic score, and CLIP; PartiPrompts (1,632 prompts) and HPDv2 (3,200 prompts) results appear in Appendix~\ref{app:full-results}. Compositional evaluation uses all 553 GenEval prompts \citep{ghosh2023geneval} and 300 prompts per T2I-CompBench category \citep{huang2023compbench}.
\paragraph{User study.} In the blind SDXL study, 16 consenting participants evaluate image pairs with the method names hidden. The study compares ToPO with Base and with Diffusion-DPO over 300 prompts, uses a shared seed, and allows ties. Participants judge overall preference, visual appeal, and prompt faithfulness. We report aggregate choice rates; Appendix~\ref{app:protocol} gives the full protocol.
\subsection{Evaluation Results}


\paragraph{Preference results.} Table~\ref{tab:main} compares final checkpoints on five automatic measures. On SD-1.5, \method{} ranks first on PickScore, HPSv2, ImageReward, and Aesthetic score. On SDXL, it ranks first on PickScore, HPSv2, ImageReward, and CLIP. Relative to Diffusion-DPO, ImageReward increases from $0.250$ to $0.812$ on SD-1.5 and from $0.934$ to $1.032$ on SDXL. HPSv2 also improves on both backbones. The matched three-seed comparison in Appendix~\ref{app:seed-stability} further confirms gains on all SD-1.5 measures and on SDXL HPSv2, ImageReward, and CLIP.

\begin{table*}[!t]
\caption{Pick-a-Pic checkpoint results on 500 prompts for SD-1.5 and SDXL.}
\label{tab:main}
\centering
\small
\setlength{\tabcolsep}{4.2pt}
\renewcommand{\arraystretch}{0.80}
\begin{tabularx}{\textwidth}{@{}ll
  >{\hsize=.94\hsize\centering\arraybackslash}X
  >{\hsize=.94\hsize\centering\arraybackslash}X
  >{\hsize=1.24\hsize\centering\arraybackslash}X
  >{\hsize=.94\hsize\centering\arraybackslash}X
  >{\hsize=.94\hsize\centering\arraybackslash}X@{}}
\toprule
Backbone & Method & PickScore $\uparrow$ & HPSv2 $\uparrow$ &
ImageReward $\uparrow$ & Aesthetic $\uparrow$ & CLIP $\uparrow$ \\
\midrule
\multirow{6}{*}{SD-1.5}
& Base & 20.543 & 0.2467 & 0.074 & 5.325 & 27.199 \\
& Diffusion-DPO & 20.947 & 0.2586 & 0.250 & 5.336 & 27.595 \\
& SFT-winners & 21.181 & 0.2824 & 0.642 & 5.354 & 28.064 \\
& KTO & 21.108 & 0.2810 & 0.618 & 5.351 & 28.003 \\
& InPO & \runnerup{21.410} & \runnerup{0.2871} & \runnerup{0.748} & \runnerup{5.357} & \best{28.467} \\
& \method{} & \best{21.542} & \best{0.2903} & \best{0.812} & \best{5.361} & \runnerup{28.393} \\
\midrule
\multirow{6}{*}{SDXL}
& Base & 21.866 & 0.2869 & 0.744 & 5.359 & 28.411 \\
& Diffusion-DPO & \runnerup{22.264} & 0.3005 & 0.934 & \runnerup{5.365} & \runnerup{29.018} \\
& SFT-winners & 21.481 & 0.2829 & 0.647 & 5.348 & 28.526 \\
& InPO & 22.208 & \runnerup{0.3044} & \runnerup{0.955} & \runnerup{5.365} & 28.842 \\
& MaPO & 21.945 & 0.2962 & 0.866 & \best{5.377} & 28.613 \\
& \method{} & \best{22.270} & \best{0.3140} & \best{1.032} & 5.360 & \best{29.487} \\
\bottomrule
\end{tabularx}
\end{table*}

\paragraph{Compositional results.} Table~\ref{tab:compositional} evaluates whether the preference improvements extend to prompts with compositional constraints. The GenEval and T2I-CompBench averages are unweighted means over the displayed categories; the T2I-CompBench average is a manuscript-derived summary of the six shared categories, not an official score. On SD-1.5, \method{} raises the GenEval average from $0.4591$ to $0.5092$ and the T2I-CompBench average from $0.3521$ to $0.4123$ relative to Diffusion-DPO. The gains cover the reported color, position, and shape categories. On SDXL, \method{} achieves the highest GenEval average and improves the reported GenEval components, while its T2I-CompBench average remains nearly tied with Diffusion-DPO.

\begingroup
\setlength{\intextsep}{5pt}
\begin{table}[H]
\caption{Compositional evaluation on GenEval and T2I-CompBench. Shading indicates point-estimate rank only.}
\label{tab:compositional}
\centering
\scriptsize
\setlength{\tabcolsep}{2.4pt}
\renewcommand{\arraystretch}{0.94}
\resizebox{\textwidth}{!}{%
\begin{tabular}{lcccccc@{\hspace{12pt}}lcccccc}
\toprule
\multicolumn{7}{c}{SD-1.5} & \multicolumn{7}{c}{SDXL}\\
\cmidrule(lr){1-7}\cmidrule(lr){8-14}
& \multicolumn{3}{c}{GenEval $\uparrow$} & \multicolumn{3}{c}{T2I-CompBench $\uparrow$} & & \multicolumn{3}{c}{GenEval $\uparrow$} & \multicolumn{3}{c}{T2I-CompBench $\uparrow$}\\
\cmidrule(lr){2-4}\cmidrule(lr){5-7}\cmidrule(lr){9-11}\cmidrule(lr){12-14}
Method & Avg.$^{\ddagger}$ & Colors & Position & Avg.$^{\dagger}$ & Color & Shape & Method & Avg.$^{\ddagger}$ & Colors & Position & Avg.$^{\dagger}$ & Color & Shape\\
\midrule
Base & 0.4378 & 0.7553 & 0.0500 & 0.3393 & 0.3750 & 0.3715 & Base & 0.5783 & 0.8723 & \runnerup{0.1275} & 0.4533 & 0.6277 & 0.4878\\
Diffusion-DPO & 0.4591 & 0.7926 & 0.0625 & 0.3521 & 0.4028 & 0.3842 & Diffusion-DPO & \runnerup{0.5917} & 0.8644 & 0.1075 & \runnerup{0.4974} & \runnerup{0.7177} & \runnerup{0.5378}\\
SFT-winners & 0.4878 & 0.8138 & \runnerup{0.0725} & 0.3870 & 0.4761 & 0.4310 & SFT-winners & 0.5464 & 0.8670 & 0.1250 & 0.4436 & 0.6089 & 0.4955\\
KTO & 0.4879 & \runnerup{0.8271} & 0.0550 & 0.3887 & 0.4875 & 0.4300 & InPO & 0.5827 & 0.8750 & 0.1125 & 0.4731 & 0.6731 & 0.5231\\
InPO & \runnerup{0.4925} & 0.8245 & \runnerup{0.0725} & \runnerup{0.3984} & \runnerup{0.4967} & \runnerup{0.4389} & MaPO & 0.5687 & \runnerup{0.8803} & 0.1225 & 0.4691 & 0.6429 & 0.5197\\
\method{} & \best{0.5092} & \best{0.8404} & \best{0.0775} & \best{0.4123} & \best{0.5242} & \best{0.4579} & \method{} & \best{0.6168} & \best{0.9069} & \best{0.1675} & \best{0.4975} & \best{0.7190} & \best{0.5398}\\
\bottomrule
\end{tabular}}
\end{table}
\endgroup
\vspace{-2pt}
\paragraph{Human evaluation.} In the blind SDXL study, 16 participants compare image pairs with the method names hidden, and ties are allowed. Figure~\ref{fig:human-study} shows that \method{} receives the majority of votes for overall preference, visual appeal, and prompt faithfulness. Its win rates are $66.0\%$, $73.3\%$, and $67.3\%$ against Base, and $63.3\%$, $70.7\%$, and $65.3\%$ against Diffusion-DPO, respectively.

\vspace{-3mm}
\begin{figure}[H]
  \centering
  \includegraphics[width=.98\textwidth]{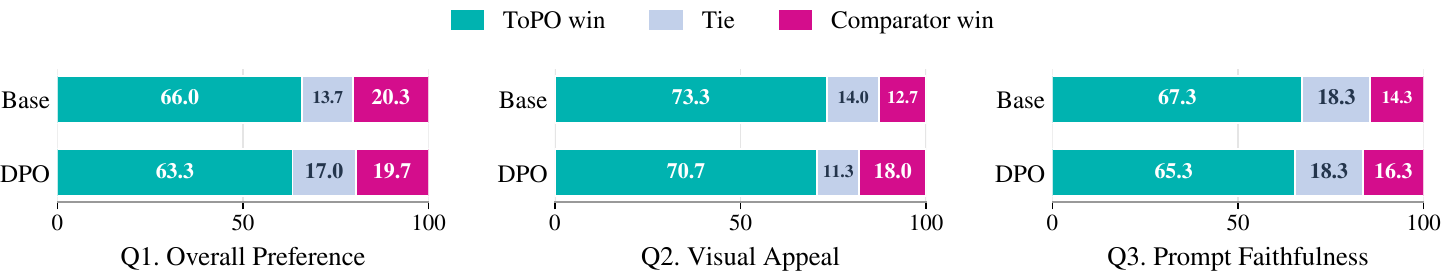}
\caption{\textbf{Blind SDXL preference study.}  Aggregate choice rates for overall preference, visual appeal, and prompt faithfulness in comparisons of ToPO with Base and Diffusion-DPO.  Bars separate ToPO wins, ties, and comparator wins.}
  \label{fig:human-study}
\end{figure}

\subsection{Ablation Studies and Controlled Route Analysis}

We use Figure~\ref{fig:controlled-trajectory} and Table~\ref{tab:ablations} to examine training length and the contributions of the route design under the locked SD-1.5 configuration.


\paragraph{Training-step ablation.}
Figure~\ref{fig:controlled-trajectory} traces one matched SD-1.5 run; the gray dashed lines show the author-released Diffusion-DPO checkpoint scores reported in Table~\ref{tab:main}. Within the matched run, \method{} surpasses Diffusion-DPO early and reaches near-best observed values on most measures by around 300 updates. It remains higher on PickScore, HPSv2, and ImageReward through 500 updates, suggesting that strong alignment can be reached before the end of the fixed schedule.


\begin{figure*}[!t]
  \centering
  \includegraphics[width=.96\textwidth]{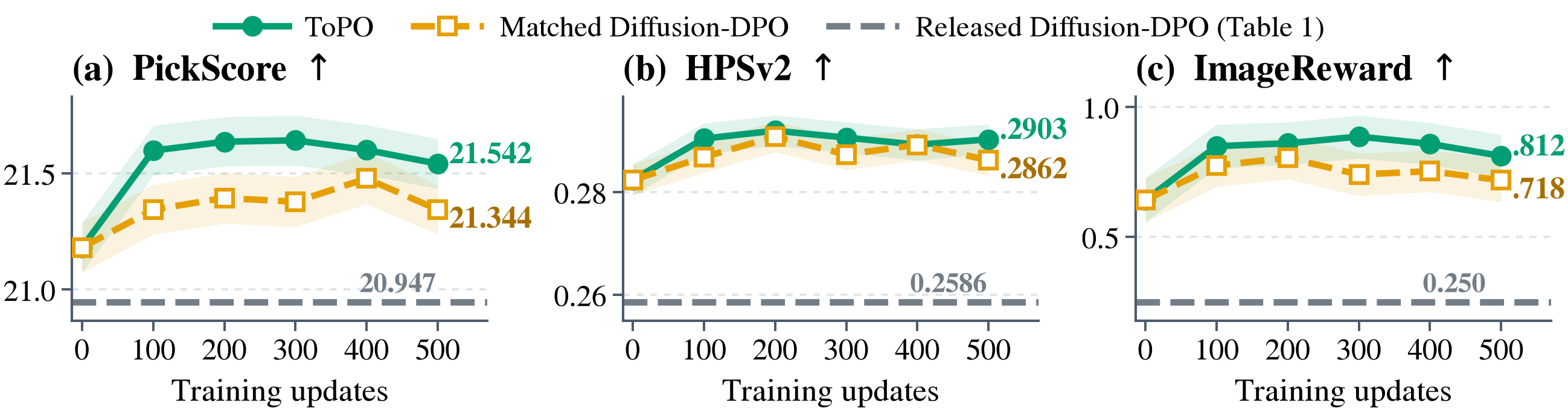}
  \caption{\textbf{Controlled SD-1.5 training dynamics.} PickScore, HPSv2, and ImageReward are evaluated at each checkpoint. Bands are 95\% prompt-bootstrap intervals on 500 prompts; seedwise endpoints are in Appendix~\ref{app:seed-stability}. Gray dashed lines show the author-released Diffusion-DPO checkpoint scores from Table~\ref{tab:main}, as fixed references distinct from the matched training trajectory.}
  \label{fig:controlled-trajectory}
\end{figure*}


\paragraph{Route and component ablations.}
Table~\ref{tab:ablations} jointly examines route allocation, individual route components, and preference-scale sensitivity. Uniform-$W$ replaces the learned route with a uniform weight while retaining the midpoint term, removing structured allocation across spatial locations, denoising steps, and tokens. Shuffled-$W$ preserves the materialized route values but randomly changes their assignment to the corresponding coordinates. The full route obtains the largest point estimate across the displayed outcomes, indicating that the gains depend on where the weights are applied rather than only on their overall magnitude. Setting the spatial, timestep, or token factor to one lowers PickScore by $0.295$, $0.307$, and $0.168$, respectively, showing that all three factors contribute to the final alignment behavior. Removing SOS/EOS masking causes a larger decrease of $0.429$, confirming that special-token attention can distort token normalization. The midpoint term improves both PickScore and GenEval, while the preference-scale sweep is non-monotonic across metrics and identifies $\beta=2000$ as the strongest overall setting.




\begingroup
\setlength{\intextsep}{0pt}
\begin{table}[H]
\caption{SD-1.5 ToPO ablations and route-allocation controls. Panel~A uses one locked run per route allocation. Shading marks point-estimate rank, not significance.}
\label{tab:ablations}
\centering
\scriptsize
\setlength{\tabcolsep}{3.4pt}
\renewcommand{\arraystretch}{0.80}
\begin{tabular}{lccccccccc}
\toprule
Configuration & PickScore $\uparrow$ & HPSv2 $\uparrow$ & ImageReward $\uparrow$ & Aesthetic $\uparrow$ & CLIP $\uparrow$ & GE Avg.$^{\ddagger}$ $\uparrow$ & GE Count $\uparrow$ & Final $r$-gap\\
\midrule
\multicolumn{9}{l}{\textit{A. Route-allocation controls (one locked run per configuration)}}\\
\method{} (full) & \best{21.542} & \best{.2903} & \best{.812} & \best{5.361} & \best{28.393} & \best{.5092} & \best{.4969} & $3.64\!\times\!10^{-4}$\\
Uniform-$W$ & 21.237 & .2851 & \runnerup{.769} & 5.348 & \runnerup{28.198} & .4710 & \runnerup{.4418} & $4.29\!\times\!10^{-4}$\\
Shuffled-$W$ & \runnerup{21.382} & \runnerup{.2899} & .741 & \runnerup{5.358} & 28.180 & \runnerup{.4926} & .4375 & $-3.47\!\times\!10^{-5}$\\
\midrule
\multicolumn{9}{l}{\textit{B. Component removals}}\\
\method{} (full) & \best{21.542} & \best{.2903} & \best{.812} & 5.361 & \runnerup{28.393} & \best{.5092} & \best{.4969} & $3.64\!\times\!10^{-4}$\\
$w_k\equiv1$ & \runnerup{21.374} & .2842 & .803 & \best{5.365} & 28.271 & .5026 & .4406 & $5.53\!\times\!10^{-4}$\\
$w_\tau\equiv1$ & 21.235 & \runnerup{.2883} & .805 & 5.361 & 28.331 & \runnerup{.5079} & .4500 & $4.27\!\times\!10^{-4}$\\
$w_s\equiv1$ & 21.247 & .2856 & \runnerup{.808} & \runnerup{5.364} & 28.371 & .5026 & .4219 & $5.34\!\times\!10^{-4}$\\
No SOS/EOS prior & 21.113 & .2845 & .788 & 5.362 & 28.311 & .5060 & \runnerup{.4531} & $6.50\!\times\!10^{-4}$\\
No triplet term & 21.295 & .2840 & .807 & 5.363 & \best{28.412} & .4950 & .4375 & $6.81\!\times\!10^{-4}$\\
\midrule
\multicolumn{9}{l}{\textit{C. Preference scale $\beta$}}\\
$500$ & \runnerup{21.484} & .2902 & .804 & \best{5.368} & 28.277 & .4894 & .4062 & $1.00\!\times\!10^{-3}$\\
$1000$ & 21.433 & \best{.2921} & \best{.856} & \runnerup{5.364} & 28.201 & .4901 & \runnerup{.4406} & $8.64\!\times\!10^{-4}$\\
$2000$ (default) & \best{21.542} & \runnerup{.2903} & \runnerup{.812} & 5.361 & \best{28.393} & \best{.5092} & \best{.4969} & $3.64\!\times\!10^{-4}$\\
$4000$ & 21.398 & .2883 & .774 & 5.362 & \runnerup{28.298} & \runnerup{.4972} & .4312 & $4.85\!\times\!10^{-4}$\\
\bottomrule
\end{tabular}
\end{table}

\endgroup
\endgroup



\begingroup
\makeatletter
\renewcommand{\section}{\@startsection{section}{1}{\z@}{0pt}{0.25ex plus 0.1ex minus 0.1ex}{\large\sc\raggedright}}
\makeatother
\section{Conclusion}

We introduce \method{}, a token-conditioned preference optimization method for diffusion models that uses preferred-branch cross-attention to relate prompt tokens to spatial residual evidence and translates each global preference comparison into reweighted local noise-prediction updates across spatial locations and denoising timesteps.  At four denoising anchors, a frozen reference denoiser measures branchwise residual contrast to construct spatial and timestep evidence without local preference labels or a learned reward model.  Cross-attention reads the pre-prior spatial evidence into content-token weights and projects the normalized token weights back to the latent grid, refining the spatial route.  The resulting detached spatial--temporal route reweights the Diffusion-DPO objective, while pixel-midpoint anchoring supplies a symmetric auxiliary ordering constraint.  Extensive experiments across preference and compositional evaluations demonstrate the effectiveness of \method{}.  \method{} provides a new perspective on diffusion preference learning by routing global preference feedback through diffusion's native spatial and temporal coordinates while using prompt tokens to identify the image content that this feedback should emphasize.
\endgroup
\clearpage
\bibliographystyle{iclr2027_conference}
\bibliography{references_verified}

\clearpage
\appendix
\renewcommand{\thesection}{S\arabic{section}}
\renewcommand{\thesubsection}{\thesection.\arabic{subsection}}
\setcounter{section}{0}
\begin{center}
{\Large\bfseries Appendix}
\end{center}

\section{AI Use Statement}

No generative AI tools were used to formulate the research questions or hypotheses, develop the methodology, design or conduct the experiments, generate or process data, analyze the results, or prepare the figures and tables reported in this paper. The scientific claims and conclusions were developed and verified by the authors, who take full responsibility for the content of the submission. A generative AI assistant was used solely to help draft the wording of this disclosure statement; the authors reviewed and finalized that wording.
\setlength{\dbltextfloatsep}{20pt plus 2pt minus 4pt}
\section{Supporting Analysis}
\label{app:formal}
\label{app:analysis}
\label{app:proofs}

This section provides supporting analyses for the detached route used by \method{}.  The first analysis models the route as a current-minibatch field whose weights are recomputed by the frozen reference and detached before policy differentiation.  The second establishes consistency of the attention-weighted token statistic when evaluated at a fixed checkpoint on held-out samples.  We also give a simple scale bound for the routed preference logit.

\subsection{Detached route field}

Let $Z_{t+1}$ contain a preference minibatch together with the diffusion randomness used to evaluate its loss, and let $\mathcal F_t$ denote the optimization history through $\theta_t$.  Because the route is recomputed by the frozen reference and detached before policy differentiation, define
\begin{equation}
 \widetilde G_t
 =
 \nabla_\theta^{\mathrm{sg}}
 \ell\!\left(
 \theta_t,Z_{t+1};
 \mathcal W_{\mathrm{ref}}(Z_{t+1})
 \right),
 \qquad
 G_t=\operatorname{Clip}_{1.0}(\widetilde G_t),
 \qquad
 h(\theta)=\E_{Z\sim P_0}[G(\theta,Z)].
 \label{eq:adaptive-field}
\end{equation}

We consider a projected-SGD surrogate with independently resampled minibatches, bounded post-clipping gradients, local dissipativity around a stable root $\theta^\dagger$, and a Robbins--Monro step-size schedule for the almost-sure convergence statement.  These assumptions describe the detached route field.  The experiments themselves use constant-step AdamW.

\begin{proposition}[Current-minibatch detached field]
\label{thm:adaptive}

Under the stated surrogate assumptions, the noise term
\[
\xi_{t+1}=G_t-h(\theta_t)
\]
is a martingale difference, and the corresponding projected iteration converges almost surely to $\theta^\dagger$.

\end{proposition}

Indeed, conditional on $\mathcal F_t$, the same current draw is used to recompute $\mathcal W_{\mathrm{ref}}(Z_{t+1})$.  Therefore,
\begin{equation}
 \E\!\left[
 G(\theta_t,Z_{t+1})
 \mid \mathcal F_t
 \right]
 =
 h(\theta_t).
 \label{eq:conditional-field}
\end{equation}

Let $e_t=\theta_t-\theta^\dagger$.  Projection nonexpansiveness and the post-clipping gradient bound yield
\begin{equation}
 \E\!\left[
 \norm{e_{t+1}}_2^2
 \mid \mathcal F_t
 \right]
 \leq
 (1-2\mu\eta_{t+1})
 \norm{e_t}_2^2
 +
 \eta_{t+1}^2G_0^2.
 \label{eq:lyapunov-recursion}
\end{equation}

A standard summable-drift argument then gives the stated convergence result.  For a constant step size $0<\eta<1/(2\mu)$, the same recursion gives the following neighborhood bound.

\begin{proposition}[Gradient-energy and constant-step bounds]
\label{thm:energy}

Suppose that the unclipped local gradient can be written as
\[
\widetilde G_t=RD(W_t)\mathbf u_t.
\]
Then
\begin{equation}
 \operatorname{Var}(G_t\mid\mathcal F_t)
 \leq
 \norm{R}_{\mathrm{op}}^2
 \E\!\left[
 \norm{W_t}_\infty^2
 \norm{\mathbf u_t}_2^2
 \mid \mathcal F_t
 \right]
 -
 \norm{h(\theta_t)}_2^2.
 \label{eq:energy-envelope}
\end{equation}

Moreover,
\begin{equation}
 \E\!\left[
 \norm{\theta_T-\theta^\dagger}_2^2
 \right]
 \leq
 (1-2\mu\eta)^T
 \norm{\theta_0-\theta^\dagger}_2^2
 +
 \frac{\eta G_0^2}{2\mu}.
 \label{eq:constant-step}
\end{equation}

\end{proposition}

The first inequality follows from
\[
\norm{RD(W_t)\mathbf u_t}_2
\leq
\norm{R}_{\mathrm{op}}
\norm{W_t}_\infty
\norm{\mathbf u_t}_2
\]
together with clipping.  The second follows by unrolling \eqref{eq:lyapunov-recursion}.  It characterizes the neighborhood induced by a constant-step SGD surrogate and does not model the moment dynamics of AdamW.

\subsection{Attention-weighted residual contrast}

Fix a checkpoint $\bar\theta$ and evaluate the detached attributor on held-out independent samples.  Let
\[
M_{\bar\theta}(Z,s)
=
\operatorname{Norm}_1[d_s(Z,s)]
\]
denote the pre-prior spatial residual contrast, and let
$\bar A^+_{k,\bar\theta}(s;Z)$ denote the aggregated winner-side attention mass for token $k$.  The corresponding population token statistic is
\begin{equation}
 \rho_k^{\mathrm{disc}}
 =
 \frac{
 \E_{Z\sim P_0}
 \!\left[
 \sum_{s\in\mathcal S}
 \bar A^+_{k,\bar\theta}(s;Z)
 M_{\bar\theta}(Z,s)
 \right]
 }{
 \E_{Z\sim P_0}
 \!\left[
 \sum_{s\in\mathcal S}
 \bar A^+_{k,\bar\theta}(s;Z)
 \right]
 }.
 \label{eq:token-contrast-target}
\end{equation}

\begin{proposition}[Consistency of the frozen token statistic]
\label{thm:token-contrast}
For a fixed finite spatial grid, independent held-out samples, integrable numerator and denominator terms, and positive expected attention mass, the self-normalized pre-prior statistic $\widehat w^{\mathrm{data}}_{k,n}$ converges in probability to $\rho_k^{\mathrm{disc}}$.
\end{proposition}

For sample $Z_i$, define
\[
X_{i,k}
 =
 \sum_s
 \bar A^+_{k,\bar\theta}(s;Z_i)
 M_{\bar\theta}(Z_i,s),
 \qquad
Y_{i,k}
 =
 \sum_s
 \bar A^+_{k,\bar\theta}(s;Z_i).
\]
The empirical statistic is
\[
\widehat w^{\mathrm{data}}_{k,n}
 =
 \frac{\sum_i X_{i,k}}{\sum_i Y_{i,k}}.
\]
The weak law of large numbers gives convergence of the numerator and denominator averages.  Since $\E[Y_{i,k}]>0$, the continuous-mapping theorem yields the stated convergence.

\subsection{Routed-logit scale}

The effect of the preference scale can be bounded directly for both a local routed contribution and an aggregate routed logit.

\begin{proposition}[$\beta$--$W$ scale envelope]
\label{prop:beta-w-scale}
For a local routed contribution $\beta W_j\delta_j$ and an aggregate routed logit, respectively,
\begin{equation}
 \left|\beta W_j\delta_j\right|
 \leq
 \beta\norm{W}_\infty|\delta_j|,
 \qquad
 \left|
 \beta\inner{W}{\boldsymbol\delta}
 \right|
 \leq
 \beta\norm{W}_\infty
 \norm{\boldsymbol\delta}_1.
 \label{eq:vector-beta-w}
\end{equation}
\end{proposition}

These bounds follow from
$|W_j|\leq\norm W_\infty$ and the triangle inequality.  They motivate monitoring $\beta\norm W_\infty$ together with the finite preference-scale sweep reported in Table~\ref{tab:ablations}.
\section{Training and Evaluation Protocol}
\label{app:training-procedure}

The SD-1.5 endpoint is obtained by full U-Net finetuning from the SFT-winners checkpoint at $512^2$, while the SDXL endpoint is obtained by full U-Net finetuning from SDXL Base at $1024^2$. Both models are trained for 500 AdamW updates with learning rate $10^{-5}$, a 200-update linear warmup, weight decay $0.01$, gradient clipping at $1.0$, gradient checkpointing, and checkpoint evaluation every 100 updates. SD-1.5 uses a per-device batch size of 4 and SDXL uses a per-device batch size of 2. With 32 accumulation steps over eight processes, the effective batch sizes are 1,024 and 512, respectively.

Master U-Net weights and route construction use fp32, while policy autocasting uses fp16 for SD-1.5 and bf16 for SDXL. The frozen reference aggregates residual contrast and preferred-branch cross-attention over the declared heads, layers, and four denoising anchors. Padding, SOS, and EOS tokens are excluded before token normalization. Algorithm~\ref{alg:topo-training} specifies the order of one minibatch update: the route is reconstructed from the current preference pairs and detached before the policy backward pass.

\begin{algorithm}[H]
\caption{One detached preference-update step for ToPO.}
\label{alg:topo-training}
\small
\begin{algorithmic}[1]
\Require Frozen reference attributor; policy $\theta$; current preference minibatch $\mathcal B_t$
\Statex \textbf{Fixed settings:} $\mathcal A=\{50,349,649,949\}$; $\beta=2000$, $\gamma=0.1$, $\alpha=0.05$
\State Run the residual-contrast route on $\mathcal B_t$ in full precision at each $\tau_i\in\mathcal A$; collect $\{D_i,A_i^+\}_{i\in\mathcal A}$
\State Construct, mask, and normalize $w_s$, $w_\tau$, and $w_k$ using Eqs.~(\ref{eq:residual-contrast})--(\ref{eq:token-modulation})
\State $W_s(s)\gets\operatorname{Norm}_{\mathrm{mean}}\!\left(w_s(s)\sum_k \bar A_k^+(s)w_k(k)\right)$
\State $W(s,\tau)\gets W_s(s)w_\tau(\tau)$
\State $(\widetilde W_s,\widetilde w_\tau)\gets\operatorname{sg}[W_s,w_\tau]$ \Comment{detach route}
\State $\widetilde W(s,\tau)\gets\widetilde W_s(s)\widetilde w_\tau(\tau)$
\State Draw diffusion noise and training times for $\mathcal B_t$; evaluate $\LossTopo$ with $\widetilde W$
\State $x^a\gets\mathcal E((x^++x^-)/2)$ for each pair; evaluate $\LossTri$ with $\widetilde W_s$ \Comment{Eq.~(\ref{eq:triplet})}
\State $\mathcal L\gets\LossTopo+\gamma\LossTri$; compute $G\gets\nabla_\theta\mathcal L$
\State $G\gets\operatorname{Clip}_{1.0}(G)$; update $\theta\gets\operatorname{AdamW}(\theta,G)$ with the locked learning-rate schedule
\State \Return updated $\theta$
\end{algorithmic}
\end{algorithm}

\subsection{Evaluation protocol}
\label{app:protocol}

Table~\ref{tab:eval-protocol} specifies the generation and scoring settings used for each benchmark. Within each backbone and benchmark block, all methods use the same sampling and evaluation schedule, while SD-1.5 and SDXL results are reported separately. GenEval and T2I-CompBench use their official prompt files without subsampling.

\begingroup
\setlength{\intextsep}{3pt plus 1pt minus 1pt}
\begin{table}[!htbp]
\centering
\caption{\textbf{Locked automatic-evaluation protocol.}  Generation is benchmark-specific and held fixed within each backbone--benchmark comparison.  Main and matched-seed endpoints use the final 500-update checkpoint; the external checkpoint rows retain their released endpoint.  ``No extra negative prompt'' means that no additional negative-prompt argument is supplied to the generation call.}
\label{tab:eval-protocol}
\footnotesize
\setlength{\tabcolsep}{3.8pt}
\renewcommand{\arraystretch}{1.16}
\begin{tabularx}{\textwidth}{@{}>{\raggedright\arraybackslash}p{.16\textwidth}>{\raggedright\arraybackslash}p{.22\textwidth}>{\raggedright\arraybackslash}X>{\raggedright\arraybackslash}p{.27\textwidth}@{}}
\toprule
\textbf{Protocol} & \makecell[l]{\textbf{Prompt / image}\\[-1pt]\textbf{schedule}} & \textbf{Generation} & \makecell[l]{\textbf{Scorers / frozen}\\[-1pt]\textbf{implementation}} \\
\midrule
Pick-a-Pic validation & 500 prompts; one image per prompt; seed $0$ & DDIMScheduler from native pipeline config; 30 steps, CFG $7.5$, no extra negative prompt; native VAE; $512^2$ (SD-1.5) / $1024^2$ (SDXL) & PickScore v1; HPS-v2 1.2.0 (v2.1 weights); ImageReward-v1.0; improved aesthetic predictor; CLIP ViT-L/14 \\
T2I-CompBench & 300 prompts per declared category; SD-1.5: 10 images/prompt with seeds 42--51; SDXL: 5 images/prompt with seeds 42--46 & DDIMScheduler from native pipeline config; 30 steps, CFG $7.5$, no extra negative prompt; native VAE; $512^2$ / $1024^2$ & Official T2I-CompBench evaluator, commit \texttt{1b709499}; reported SDXL categories exclude the 10-image complex evaluator \\
GenEval & 553 official prompts; four sequential samples/prompt from seed 42 & DDIMScheduler from native pipeline config; 50 steps, CFG $9.0$, no extra negative prompt; native VAE; $512^2$ / $1024^2$ & Official GenEval scoring with the mmdetection-v3.3.0 adapter \\
\bottomrule
\end{tabularx}
\end{table}

\endgroup
\FloatBarrier

\section{Matched Control Studies}
\label{app:seed-stability}

For each backbone and seed, \method{} and Diffusion-DPO share the initial checkpoint, 500-update budget, preference-stream order, effective batch size, optimizer, learning-rate schedule, and $\beta=2000$.  All results are evaluated at the final 500-update checkpoint.

\begin{table*}[t]
\centering
\caption{\textbf{Matched three-seed comparison with Diffusion-DPO.} Mean $\pm$ sample standard deviation over three locked $\beta=2000$ training-stream pairs, evaluated at the final 500-update checkpoint on the fixed 500-prompt protocol. $\Delta$ is computed for each seed as \method{} minus Diffusion-DPO. The SD-1.5 and SDXL blocks use separate protocols.}
\label{tab:seed-stability}
\vspace{2pt}
\begingroup
\footnotesize
\setlength{\tabcolsep}{2.2pt}
\renewcommand{\arraystretch}{1.18}
\resizebox{\textwidth}{!}{%
\begin{tabular}{llccccc}
\toprule
Backbone & Method / effect & PickScore $\uparrow$ & HPSv2 $\uparrow$ & ImageReward $\uparrow$ & Aesthetic $\uparrow$ & CLIP $\uparrow$ \\
\midrule
\multirow{3}{*}{SD-1.5}
& Diffusion-DPO
& 20.930 $\pm$ 0.057
& 0.2587 $\pm$ 0.0006
& 0.250 $\pm$ 0.001
& 5.340 $\pm$ 0.012
& 27.609 $\pm$ 0.070 \\
& \method{}
& 21.549 $\pm$ 0.059
& 0.2895 $\pm$ 0.0014
& 0.813 $\pm$ 0.015
& 5.359 $\pm$ 0.013
& 28.365 $\pm$ 0.032 \\
& \makecell[l]{\textbf{Paired $\Delta$}\\[-1pt]\textbf{(\method{}--DPO)}}
& \textbf{+0.620 $\pm$ 0.096}
& \textbf{+0.0307 $\pm$ 0.0019}
& \textbf{+0.563 $\pm$ 0.016}
& \textbf{+0.019 $\pm$ 0.006}
& \textbf{+0.756 $\pm$ 0.038} \\
\addlinespace[3pt]
\midrule
\multirow{3}{*}{SDXL}
& Diffusion-DPO
& 22.285 $\pm$ 0.053
& 0.3003 $\pm$ 0.0007
& 0.932 $\pm$ 0.000
& 5.370 $\pm$ 0.012
& 28.992 $\pm$ 0.066 \\
& \method{}
& 22.253 $\pm$ 0.030
& 0.3135 $\pm$ 0.0020
& 1.030 $\pm$ 0.020
& 5.355 $\pm$ 0.007
& 29.499 $\pm$ 0.109 \\
& \makecell[l]{\textbf{Paired $\Delta$}\\[-1pt]\textbf{(\method{}--DPO)}}
& \textbf{-0.032 $\pm$ 0.026}
& \textbf{+0.0132 $\pm$ 0.0016}
& \textbf{+0.098 $\pm$ 0.020}
& \textbf{-0.015 $\pm$ 0.012}
& \textbf{+0.507 $\pm$ 0.159} \\
\bottomrule
\end{tabular}%
}
\endgroup
\end{table*}

On SD-1.5, all five paired means favor \method{}.  On SDXL, the paired means favor \method{} on HPSv2, ImageReward, and CLIP, while Diffusion-DPO is higher on PickScore and Aesthetic.  The table reports the mean and sample standard deviation across the three matched training pairs.
\section{Extended Results and Qualitative Evidence}
\label{app:full-results}


The main-text comparison evaluates Pick-a-Pic on 500 prompts. Here, we extend the evaluation of \method{} and the within-backbone baselines to GenEval, T2I-CompBench, PartiPrompts, and HPDv2 under the locked protocol. GenEval includes 553 prompts, T2I-CompBench uses 300 prompts per category, and the held-out PartiPrompts and HPDv2 suites contain 1,632 and 3,200 prompts, respectively. The qualitative figures below complement these numerical comparisons.

\newcommand{\TopoExpandedCompBenchTable}{%
\caption{Expanded T2I-CompBench results.}
\label{tab:app-compbench}
{\scriptsize
\setlength{\tabcolsep}{3.0pt}
\renewcommand{\arraystretch}{0.94}
\resizebox{\textwidth}{!}{%
\begin{tabular}{llccccccc}
\toprule
Backbone & Method & CB Avg.$^{\dagger}$ $\uparrow$ & Color $\uparrow$ & Shape $\uparrow$ & Texture $\uparrow$ & Count $\uparrow$ & Spatial $\uparrow$ & Non-spatial $\uparrow$\\
\midrule
\multirow{6}{*}{SD-1.5}
& Base & 0.3393 & 0.3750 & 0.3715 & 0.4201 & 0.4514 & 0.1068 & 0.3109\\
& Diffusion-DPO & 0.3521 & 0.4028 & 0.3842 & 0.4243 & 0.4591 & 0.1309 & 0.3114\\
& SFT-winners & 0.3870 & 0.4761 & 0.4310 & 0.4773 & 0.4677 & 0.1572 & 0.3129\\
& KTO & 0.3887 & 0.4875 & 0.4300 & 0.4729 & 0.4748 & 0.1538 & 0.3131\\
& InPO & \runnerup{0.3984} & \runnerup{0.4967} & \runnerup{0.4389} & \runnerup{0.4966} & \runnerup{0.4785} & \runnerup{0.1663} & \runnerup{0.3135}\\
& \method{} & \best{0.4123} & \best{0.5242} & \best{0.4579} & \best{0.5175} & \best{0.4793} & \best{0.1816} & \best{0.3136}\\
\midrule
\multirow{6}{*}{SDXL}
& Base & 0.4533 & 0.6277 & 0.4878 & 0.5587 & 0.5114 & 0.2179 & 0.3161\\
& Diffusion-DPO & \runnerup{0.4974} & \runnerup{0.7177} & \runnerup{0.5378} & \runnerup{0.6436} & \best{0.5402} & \best{0.2262} & 0.3190\\
& SFT-winners & 0.4436 & 0.6089 & 0.4955 & 0.5437 & 0.5046 & 0.1888 & \best{0.3202}\\
& InPO & 0.4731 & 0.6731 & 0.5231 & 0.5902 & 0.5176 & 0.2170 & 0.3174\\
& MaPO & 0.4691 & 0.6429 & 0.5197 & 0.5808 & 0.5302 & 0.2235 & 0.3175\\
& \method{} & \best{0.4975} & \best{0.7190} & \best{0.5398} & \best{0.6441} & \runnerup{0.5369} & \runnerup{0.2250} & \runnerup{0.3201}\\
\bottomrule
\end{tabular}}}
}

\newcommand{\TopoExpandedGenEvalTable}{%
\caption{Expanded GenEval results.  GE Avg.$^{\ddagger}$ is the unweighted mean of the six displayed GenEval components.}
\label{tab:app-geneval}
{\footnotesize
\setlength{\tabcolsep}{2.6pt}
\renewcommand{\arraystretch}{1.06}
\resizebox{\textwidth}{!}{%
\begin{tabular}{llccccccc}
\toprule
Backbone & Method & GE Avg.$^{\ddagger}$ $\uparrow$ & Single $\uparrow$ & Two $\uparrow$ & Count $\uparrow$ & Colors $\uparrow$ & Position $\uparrow$ & Color-attr. $\uparrow$\\
\midrule
\multirow{6}{*}{SD-1.5}
& Base & 0.4378 & 0.9781 & 0.4116 & 0.3719 & 0.7553 & 0.0500 & 0.0600\\
& Diffusion-DPO & 0.4591 & 0.9812 & 0.4343 & 0.3938 & 0.7926 & 0.0625 & 0.0900\\
& SFT-winners & 0.4878 & \best{0.9875} & \runnerup{0.5404} & 0.4125 & 0.8138 & \runnerup{0.0725} & 0.1000\\
& KTO & 0.4879 & \best{0.9875} & 0.5202 & \runnerup{0.4375} & \runnerup{0.8271} & 0.0550 & 0.1000\\
& InPO & \runnerup{0.4925} & \runnerup{0.9844} & 0.5379 & 0.4281 & 0.8245 & \runnerup{0.0725} & \best{0.1075}\\
& \method{} & \best{0.5092} & \best{0.9875} & \best{0.5480} & \best{0.4969} & \best{0.8404} & \best{0.0775} & \runnerup{0.1050}\\
\midrule
\multirow{6}{*}{SDXL}
& Base & 0.5783 & \runnerup{0.9906} & 0.7955 & 0.4188 & 0.8723 & \runnerup{0.1275} & 0.2650\\
& Diffusion-DPO & \runnerup{0.5917} & 0.9875 & \best{0.8434} & 0.4500 & 0.8644 & 0.1075 & \best{0.2975}\\
& SFT-winners & 0.5464 & \runnerup{0.9906} & 0.7273 & 0.3562 & 0.8670 & 0.1250 & 0.2125\\
& InPO & 0.5827 & 0.9781 & 0.8106 & \runnerup{0.4750} & 0.8750 & 0.1125 & 0.2450\\
& MaPO & 0.5687 & 0.9875 & 0.7904 & 0.3812 & \runnerup{0.8803} & 0.1225 & 0.2500\\
& \method{} & \best{0.6168} & \best{1.0000} & \runnerup{0.8409} & \best{0.4920} & \best{0.9069} & \best{0.1675} & \runnerup{0.2935}\\
\bottomrule
\end{tabular}}}%
}

Table~\ref{tab:app-geneval} reports the GenEval components, and Table~\ref{tab:app-parti} reports held-out PartiPrompts scores. Table~\ref{tab:app-compbench} expands the T2I-CompBench results by category, while Table~\ref{tab:app-hpd} reports HPDv2 scores. Table~\ref{tab:app-wtoken} presents a separate frozen-reference token-modulation diagnostic on 300 controlled cases.

\begin{table*}[t]
\centering
\setlength{\abovecaptionskip}{2pt}
\setlength{\belowcaptionskip}{0pt}

\TopoExpandedGenEvalTable

\vspace{10pt}

\caption{Held-out PartiPrompts results on 1,632 prompts using one random training seed.}
\label{tab:app-parti}
{\footnotesize
\setlength{\tabcolsep}{4.2pt}
\renewcommand{\arraystretch}{1.00}
\resizebox{\textwidth}{!}{%
\begin{tabular}{llccccc}
\toprule
Backbone & Method & PickScore $\uparrow$ & HPSv2 $\uparrow$ & ImageReward $\uparrow$ & Aesthetic $\uparrow$ & CLIP $\uparrow$\\
\midrule
\multirow{6}{*}{SD-1.5}
& Base & 21.275 & 0.2513 & 0.205 & 5.318 & 27.103\\
& Diffusion-DPO & 21.509 & 0.2592 & 0.359 & 5.328 & 27.373\\
& SFT-winners & 21.626 & 0.2804 & 0.653 & 5.354 & 27.754\\
& KTO & 21.611 & 0.2795 & 0.618 & 5.351 & 27.734\\
& InPO & \runnerup{21.785} & \runnerup{0.2839} & \runnerup{0.736} & \runnerup{5.355} & \runnerup{28.064}\\
& \method{} & \best{21.858} & \best{0.2894} & \best{0.829} & \best{5.361} & \best{28.081}\\
\midrule
\multirow{6}{*}{SDXL}
& Base & 22.375 & 0.2790 & 0.816 & 5.343 & 28.062\\
& Diffusion-DPO & \runnerup{22.635} & 0.2938 & \best{1.096} & \runnerup{5.358} & \best{28.655}\\
& SFT-winners & 21.974 & 0.2772 & 0.787 & 5.338 & 28.529\\
& InPO & \best{22.657} & \runnerup{0.2953} & 1.027 & 5.350 & 28.498\\
& MaPO & 22.410 & 0.2868 & 0.935 & \best{5.364} & 28.190\\
& \method{} & \best{22.657} & \best{0.2954} & \runnerup{1.064} & 5.357 & \runnerup{28.604}\\
\bottomrule
\end{tabular}}}
\end{table*}

\begin{table*}[t]
\centering
\setlength{\abovecaptionskip}{2pt}
\setlength{\belowcaptionskip}{0pt}

\TopoExpandedCompBenchTable

\vspace{5pt}

\caption{Held-out HPDv2 results on 3,200 prompts using one random training seed.}
\label{tab:app-hpd}
{\scriptsize
\setlength{\tabcolsep}{4.2pt}
\renewcommand{\arraystretch}{0.92}
\resizebox{\textwidth}{!}{%
\begin{tabular}{llccccc}
\toprule
Backbone & Method & PickScore $\uparrow$ & HPSv2 $\uparrow$ & ImageReward $\uparrow$ & Aesthetic $\uparrow$ & CLIP $\uparrow$\\
\midrule
\multirow{6}{*}{SD-1.5}
& Base & 20.825 & 0.2416 & 0.105 & 5.327 & 29.343\\
& Diffusion-DPO & 21.272 & 0.2551 & 0.345 & 5.339 & 29.703\\
& SFT-winners & 21.605 & 0.2865 & 0.770 & 5.357 & 30.253\\
& KTO & 21.526 & 0.2842 & 0.725 & 5.355 & 29.947\\
& InPO & \runnerup{21.854} & \runnerup{0.2907} & \runnerup{0.847} & \runnerup{5.362} & \best{30.600}\\
& \method{} & \best{21.901} & \best{0.2957} & \best{0.907} & \best{5.364} & \runnerup{30.515}\\
\midrule
\multirow{6}{*}{SDXL}
& Base & 22.547 & 0.2881 & 0.934 & 5.356 & 30.281\\
& Diffusion-DPO & 22.884 & 0.3049 & \runnerup{1.122} & \runnerup{5.363} & \runnerup{30.734}\\
& SFT-winners & 22.040 & 0.2877 & 0.872 & 5.347 & 30.434\\
& InPO & \runnerup{22.928} & \best{0.3113} & 1.095 & 5.361 & 30.637\\
& MaPO & 22.617 & 0.2988 & 1.023 & \best{5.371} & 30.464\\
& \method{} & \best{22.941} & \runnerup{0.3063} & \best{1.123} & 5.360 & \best{30.821}\\
\bottomrule
\end{tabular}}}

\vspace{4pt}
\caption{Token-modulation diagnostic with the frozen SD-1.5 reference attributor (60 cases per class).}
\label{tab:app-wtoken}
{\footnotesize
\setlength{\tabcolsep}{3.8pt}
\renewcommand{\arraystretch}{1.00}
\resizebox{\textwidth}{!}{%
\begin{tabular}{lcccccc}
\toprule
Class & Mean $R$ & Ratio of means & 95\% CI & Raw $p$ & Holm $p$ & $R>1/R<1$\\
\midrule
Color & 1.404 & 1.402 & [1.286, 1.523] & $3.4\!\times\!10^{-7}$ & $1.7\!\times\!10^{-6}$ & 48/12\\
Shape & 1.261 & 1.256 & [1.136, 1.390] & $1.1\!\times\!10^{-3}$ & $5.7\!\times\!10^{-3}$ & 39/21\\
Texture & 1.310 & 1.290 & [1.205, 1.419] & $8.9\!\times\!10^{-6}$ & $4.4\!\times\!10^{-5}$ & 43/17\\
Count & .941 & .936 & [.851, 1.032] & .140 & .140 & 23/37\\
Spatial & .918 & .906 & [.817, 1.024] & .140 & .140 & 24/36\\
\bottomrule
\end{tabular}}}

\end{table*}

\begin{figure*}[t]
  \centering
  \includegraphics[width=\textwidth,height=.86\textheight,keepaspectratio]{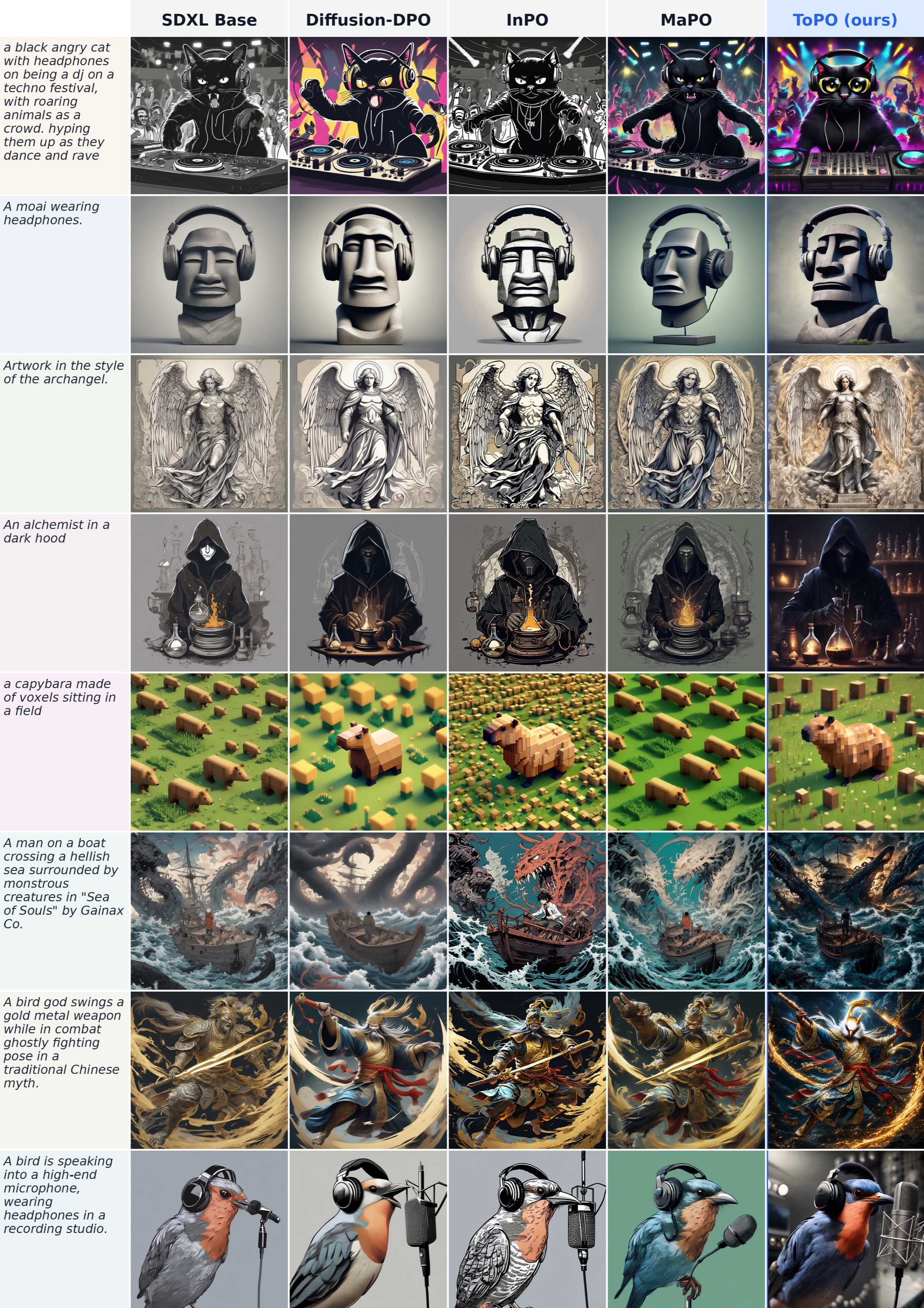}
  \caption{Additional fixed-prompt SDXL aesthetic comparisons.  Columns compare SDXL Base, Diffusion-DPO, InPO, MaPO, and \method{} on stylized subjects, visual detail, material rendering, and complex scenes.}
  \label{fig:app-aesthetics}
\end{figure*}

\begin{figure*}[t]
  \centering
  \includegraphics[width=\textwidth,height=.86\textheight,keepaspectratio]{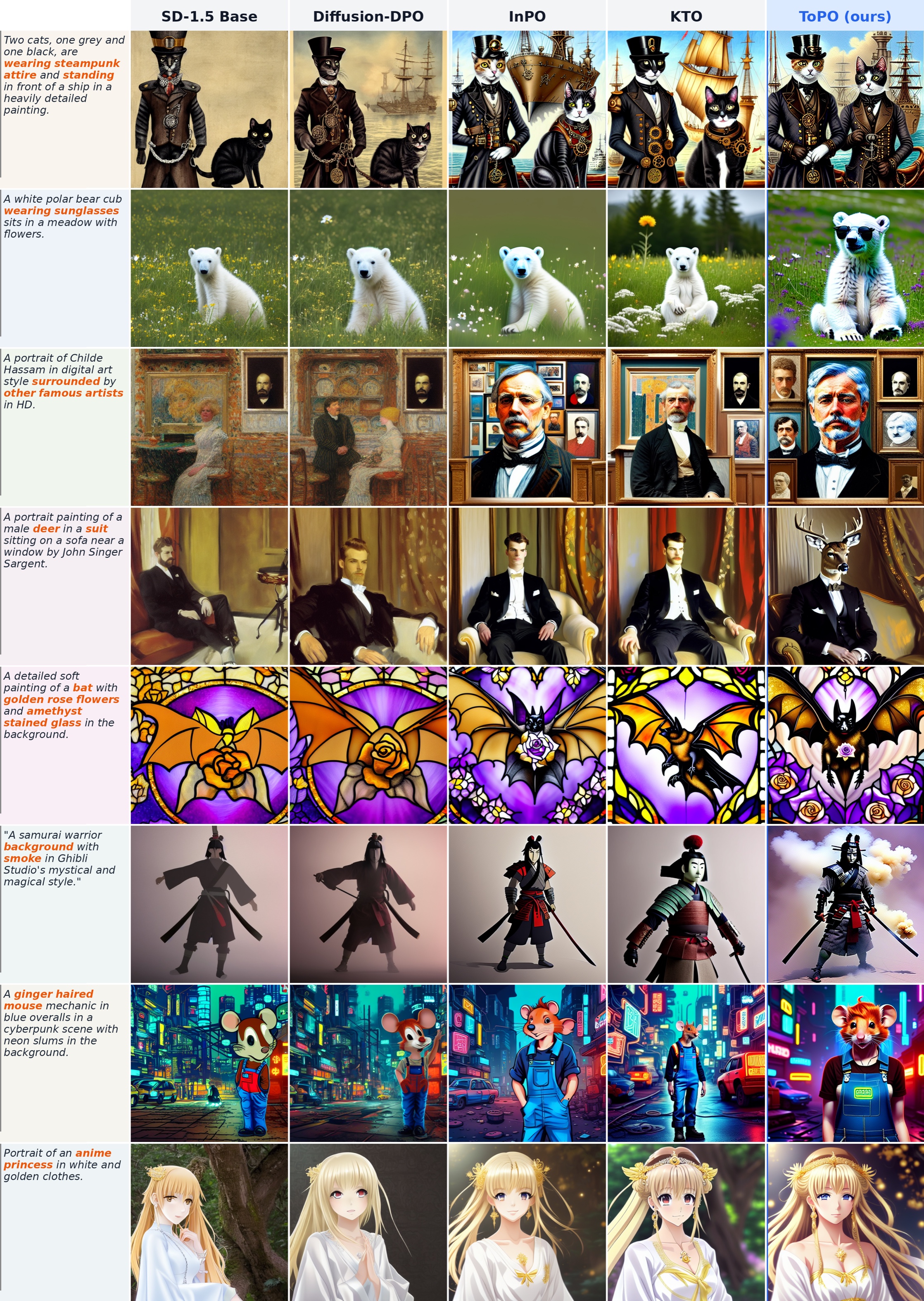}
  \caption{Additional fixed-prompt SD-1.5 visual-quality comparisons.  Columns compare SD-1.5 Base, Diffusion-DPO, InPO, KTO, and \method{} on eight prompts with detailed apparel, accessories, materials, stylized subjects, and scenes.}
  \label{fig:app-aesthetics-sd15}
\end{figure*}

\begin{figure*}[t]
  \centering
  \includegraphics[width=\textwidth,height=.86\textheight,keepaspectratio]{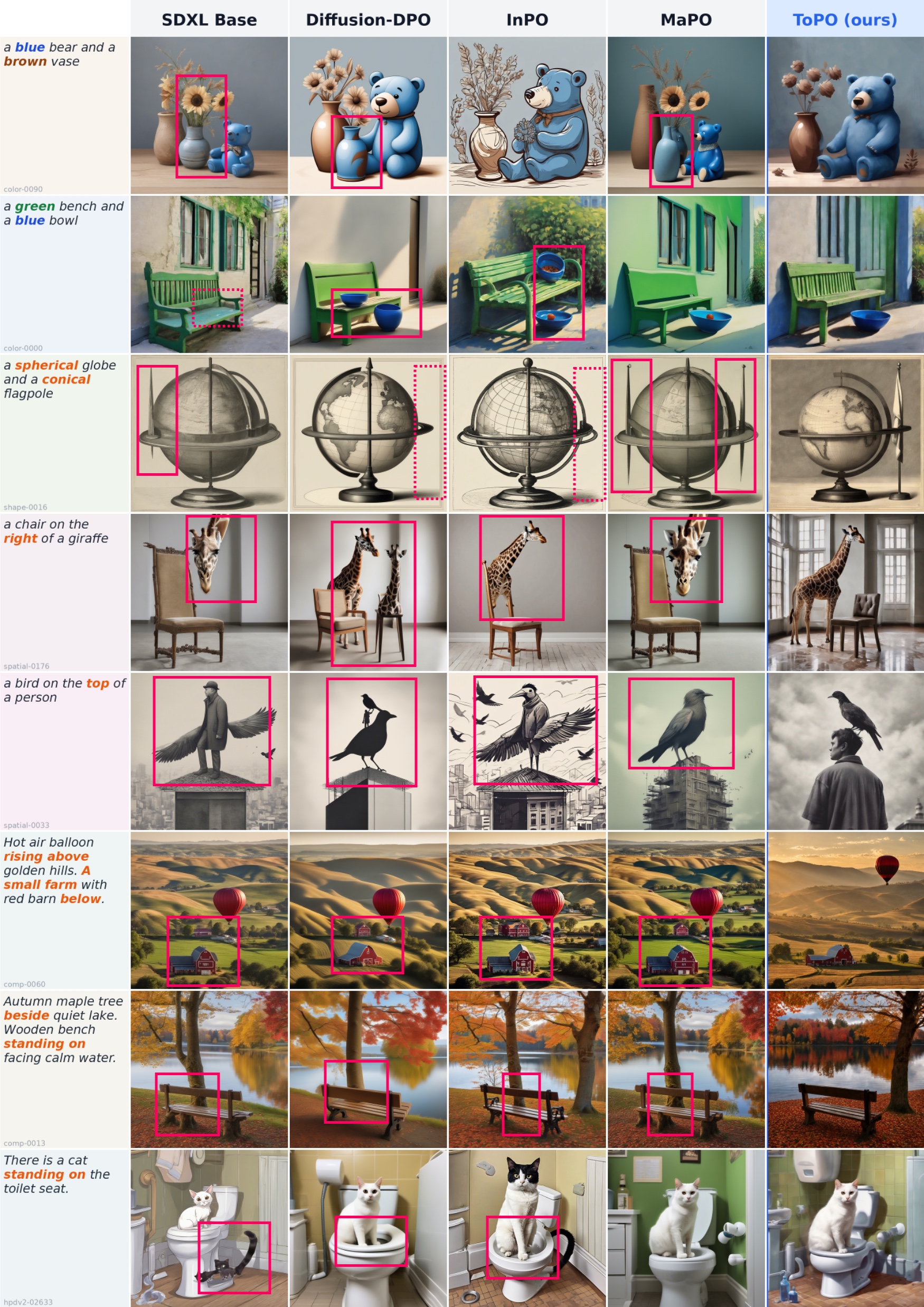}
  \caption{Additional fixed-prompt SDXL attribute and relational comparisons.  Highlighted words denote requested colors, shapes, or relations.  Columns compare SDXL Base, Diffusion-DPO, InPO, MaPO, and \method{}.}
  \label{fig:app-composition}
\end{figure*}

\begin{figure*}[t]
  \centering
  \includegraphics[width=\textwidth,height=.86\textheight,keepaspectratio]{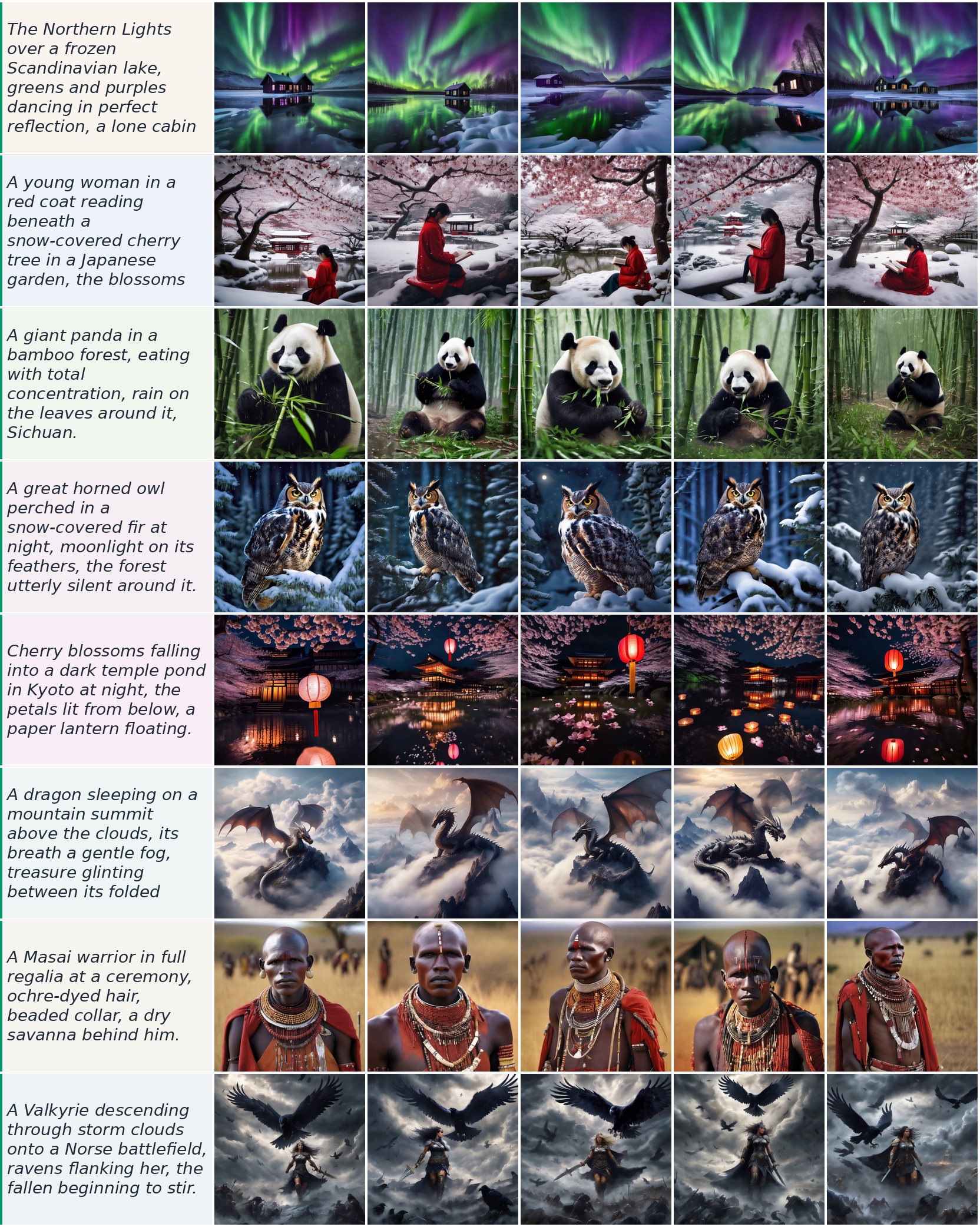}
  \caption{\textbf{Seed-wise samples from the final SDXL \method{} checkpoint.}  Each row fixes one prompt, while the five image columns use different random seeds.  The plate illustrates variation in composition, viewpoint, and local detail across seeds.}
  \label{fig:app-seed-variation}
\end{figure*}

\setlength{\dblfloatsep}{5pt}
\setlength{\textfloatsep}{5pt}
\renewcommand{\dbltopfraction}{0.55}
\renewcommand{\dblfloatpagefraction}{0.90}

\begin{figure*}[!t]
  \centering
  \includegraphics[
    width=.94\textwidth,
    height=.43\textheight,
    keepaspectratio
  ]{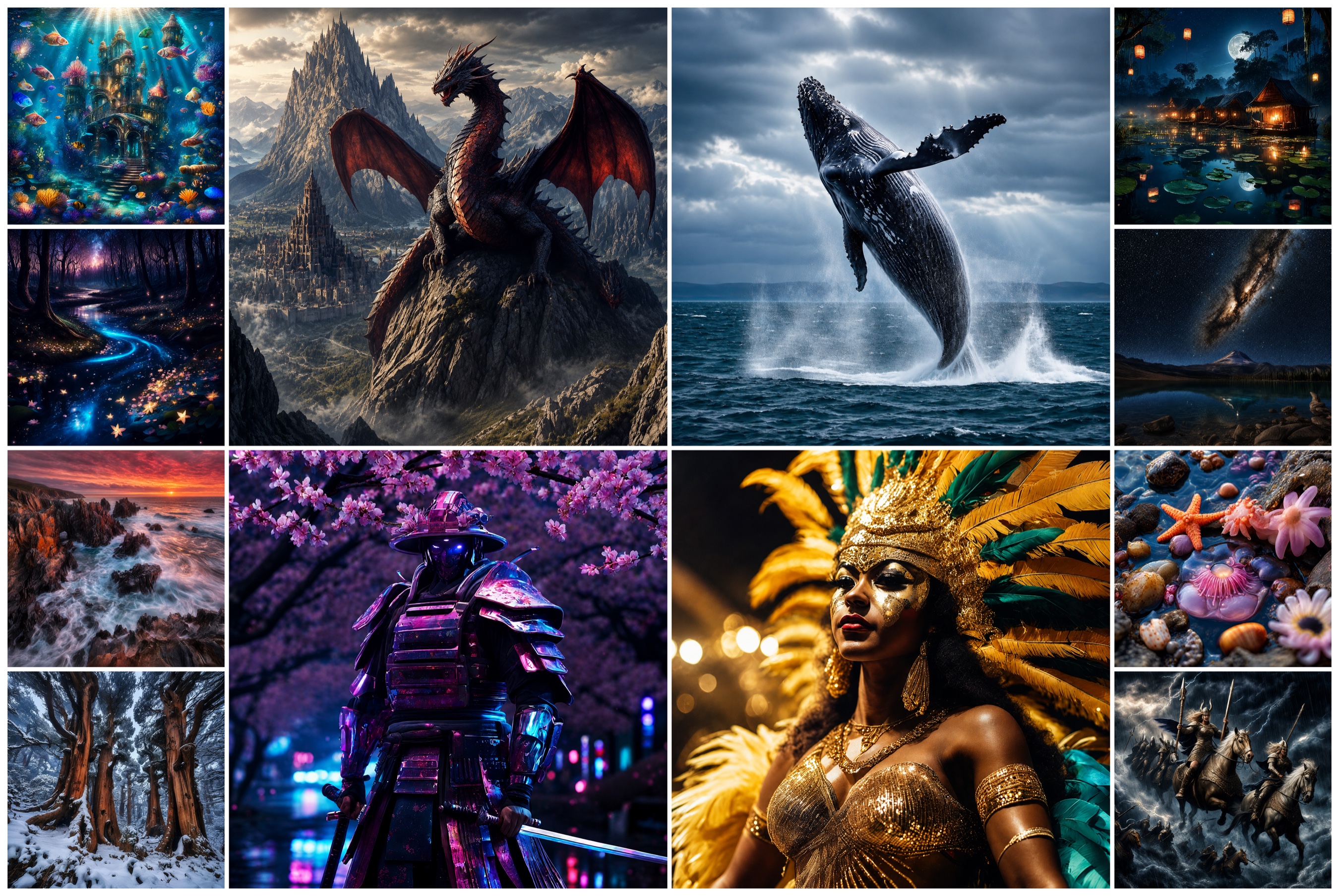}
  \caption{\textbf{Additional visual-quality samples from the final SDXL \method{} checkpoint.}
  The gallery covers detailed environments, atmospheric scenes, fantasy subjects, portraits, and material-rich natural elements.}
  \label{fig:app-visual-gallery}
\end{figure*}


\begin{figure*}[!b]
  \centering
  \includegraphics[
    width=.98\textwidth
  ]{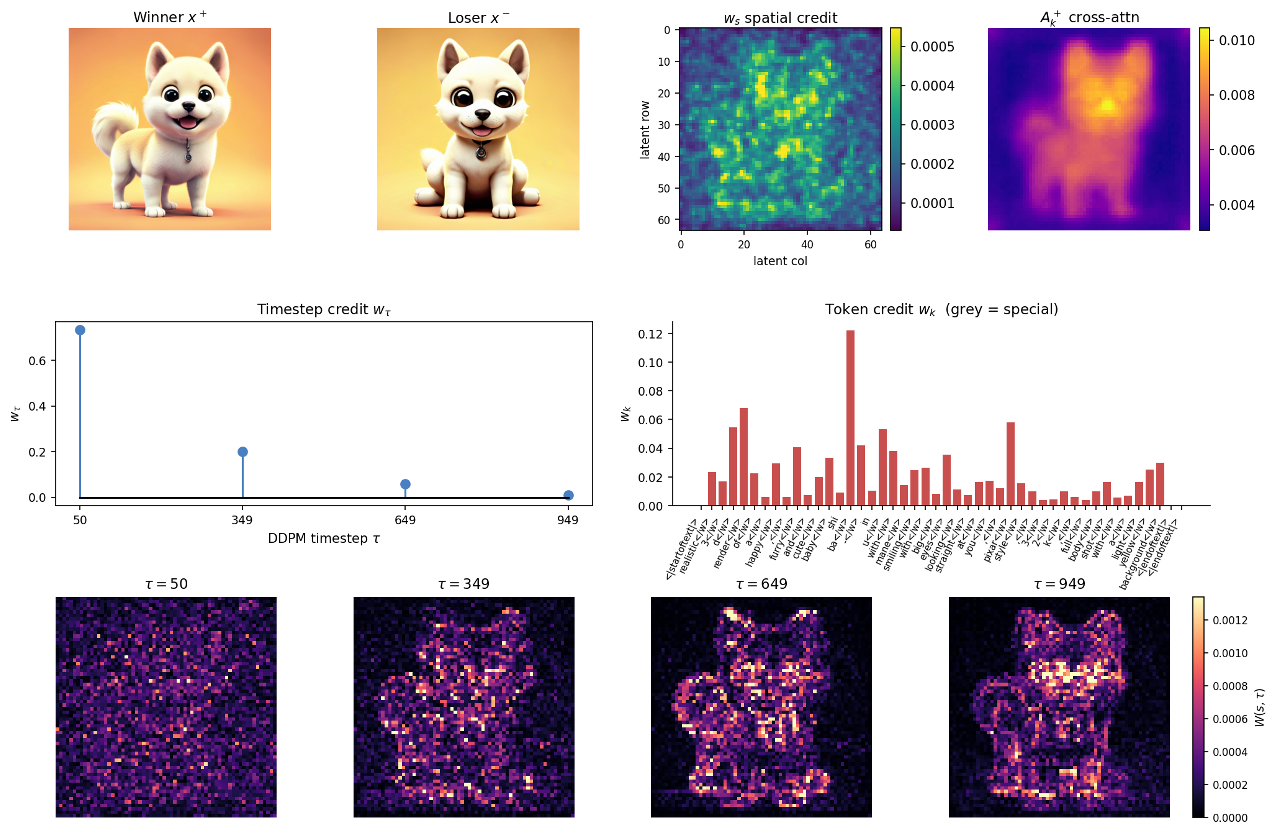}
  \caption{Token-conditioned spatial--temporal routing in \method{}.
  The held-out Pick-a-Pic v2 winner--loser pair is conditioned on the prompt
  ``Realistic 3d render of a happy, furry and cute baby Shiba-Inu with mane
  smiling with big eyes looking straight at you, Pixar style, 32k, full body
  shot with a light yellow background.''  The figure shows the spatial factor
  $w_s$, timestep factor $w_\tau$, content-token modulation $w_k$, and composed
  field $W(s,\tau)$ at four denoising anchors.  This final-checkpoint diagnostic
  is distinct from the frozen-reference route used during training.}
  \label{fig:route-trace}
\end{figure*}

\subsection{Token-Modulation and Scale Diagnostics}
\label{app:token-details}

The fixed-template controlled token test contains 300 cases, with 60 cases each for color, shape, texture, count, and spatial relations.  It compares the target-token $w_k$ mass with matched non-target content-token controls using class-wise paired analyses.  Table~\ref{tab:app-wtoken} reports this diagnostic for the frozen \texttt{sft\_\allowbreak init\_\allowbreak ref\_\allowbreak unet} reference attributor.  It therefore characterizes the attribution rule used by the diagnostic rather than the localization behavior of the final trained checkpoint.

\paragraph{A single-pair view of token-conditioned spatial--temporal routing.}
\label{app:route-trace-text}

The controlled token test characterizes a fixed reference attributor over a prompt suite.  Figure~\ref{fig:route-trace} instead shows a final-checkpoint diagnostic for one held-out SD-1.5 Pick-a-Pic v2 preference pair.  In this example, the spatial factor is nonuniform, the timestep factor assigns different weights to the denoising anchors, and the attention-weighted token readout projects content-token modulation back to the latent grid.  The resulting $W(s,\tau)$ is therefore a nonuniform spatial--temporal multiplier for this pair.

\vspace*{0.5em}

\section{Scope and Limitations}
\label{app:scope}

\method{} is designed for attention-equipped latent diffusion models that predict noise and expose both local residuals and cross-attention maps.  Applying the method to flow-matching models or other conditioning interfaces would require a compatible allocation statistic and dedicated evaluation.

The matched three-seed study covers only the locked SD-1.5 and SDXL configurations, while comparisons with released checkpoints are descriptive rather than matched retraining experiments.  The Shuffled-$W$ control examines the effect of coordinate assignment but does not establish superiority under a matched compute budget, semantic localization, or an exhaustive comparison with alternative saliency constructions.  The token diagnostic is computed with a frozen SD-1.5 reference attributor and should therefore be interpreted as a diagnostic of the attribution rule rather than a direct measure of localization in the final trained model.
\end{document}